\documentclass{article}

\usepackage[preprint, nonatbib] {neurips_2026}  
\usepackage[numbers, sort&compress]{natbib}      

\usepackage[utf8]{inputenc}
\usepackage[T1]{fontenc}
\usepackage{hyperref}
\usepackage{url}
\usepackage{booktabs}
\usepackage{amsfonts}
\usepackage{amssymb}
\usepackage{amsmath}
\usepackage{nicefrac}
\usepackage{microtype}
\usepackage{xcolor}
\usepackage{graphicx}
\usepackage{algorithm}
\usepackage{algorithmic}
\usepackage{float}  
\usepackage{multirow}
\usepackage{makecell}
\usepackage{subcaption}
\usepackage{wrapfig}
\usepackage{colortbl}
\usepackage{enumitem}
\usepackage{amsthm}
\usepackage{fontawesome5}  

\newtheorem{proposition}{Proposition}
\newtheorem{lemma}{Lemma}

\newtheorem{definition}{Definition}
\newtheorem{remark}{Remark}

\makeatletter
\renewcommand\paragraph{%
  \@startsection{paragraph}{4}{\z@}%
    {1.5ex \@plus 0.5ex \@minus 0.2ex}%
    {-0.5em}%
    {\normalfont\normalsize\bfseries}}
\renewcommand\section{%
  \@startsection{section}{1}{\z@}%
    {2.0ex \@plus 0.5ex \@minus 0.2ex}%
    {1.0ex \@plus 0.2ex}%
    {\normalfont\Large\bfseries}}
\renewcommand\subsection{%
  \@startsection{subsection}{2}{\z@}%
    {1.5ex \@plus 0.4ex \@minus 0.2ex}%
    {0.6ex \@plus 0.2ex}%
    {\normalfont\large\bfseries}}
\makeatother

\newcommand{\ours}{\textsc{MAD-OPD}}

\newcommand{\jsd}{\mathrm{JSD}}
\newcommand{\kl}{\mathrm{KL}}
\newcommand{\E}{\mathbb{E}}

\title{MAD-OPD: Breaking the Ceiling in On-Policy Distillation via Multi-Agent Debate}

\author{%
  Jianze Wang$^{1,2,*}$, \quad
  Ying Liu$^{2}$, \quad
  Jinlong Chen$^{2}$, \quad
  Xuchun Hu$^{2}$, \quad
  Qilong Zhang$^{2}$, \quad
  \\[3pt]
  \bfseries
  Yu Cao$^{2}$, \quad
  Jun Wang$^{2}$, \quad
  Hua Yang$^{2}$, \quad
  Yong Xie$^{1,\dagger}$, \quad
  Qianglong Chen$^{2,\dagger}$
  \\[3pt]
  \mdseries
  $^{1}$School of Artificial Intelligence and Automation, Huazhong University of Science and Technology
  \\[1pt]
  $^{2}$Alibaba Group
  \\[1pt]
  \faEnvelope\;\texttt{swordwang@hust.edu.cn},\;
  \texttt{qianglong.cql@alibaba-inc.com}
}

\begin{document}

\maketitle

\renewcommand{\thefootnote}{\fnsymbol{footnote}}
\footnotetext[1]{Work done during research internship at Alibaba Group.}
\footnotetext[2]{Corresponding author.}
\renewcommand{\thefootnote}{\arabic{footnote}}
\setcounter{footnote}{0}

\begin{abstract}
On-policy distillation (OPD) trains a student on its own trajectories under token-level teacher supervision, but existing methods are capped by a \emph{single-teacher capability ceiling}: when the teacher errs, the student inherits the error. OPD also remains largely unexplored in agentic tasks, where per-step errors compound across long trajectories and destabilize training. We propose \textbf{\ours{}} (\textbf{M}ulti-\textbf{A}gent \textbf{D}ebate-driven \textbf{O}n-\textbf{P}olicy \textbf{D}istillation), which breaks this ceiling by recasting the distillation teacher as a deliberative collective of teachers that debate over the student's on-policy state; the debate produces an \emph{emergent collective intelligence} that supplies token-level supervision, with each teacher's contribution weighted by its post-debate confidence. To extend OPD to agentic tasks, we also introduce On-Policy Agentic Distillation (OPAD), which adds step-level sampling to stabilize training under multi-step error compounding. We additionally derive a task-adaptive divergence principle, selecting JSD (Jensen--Shannon divergence) for agentic stability and reverse KL (Kullback--Leibler) divergence for code generation, and verify it both theoretically and empirically. Across six teacher--student configurations (Qwen3 and Qwen3.5; 1.7B--14B students, 8B--32B teachers) and five agentic and code benchmarks, \ours{} ranks first across all six configurations; on the 14B+8B$\to$4B setting it lifts the agentic average by $+2.4\%$ and the code average by $+3.7\%$ over the stronger single-teacher OPD. Code is available at: \faGithub\; \url{https://github.com/chiefovoavicii/MAD-OPD}.
\end{abstract}

\section{Introduction}
\label{sec:intro}

\begin{figure}[t]
  \centering
  \includegraphics[width=0.85\linewidth]{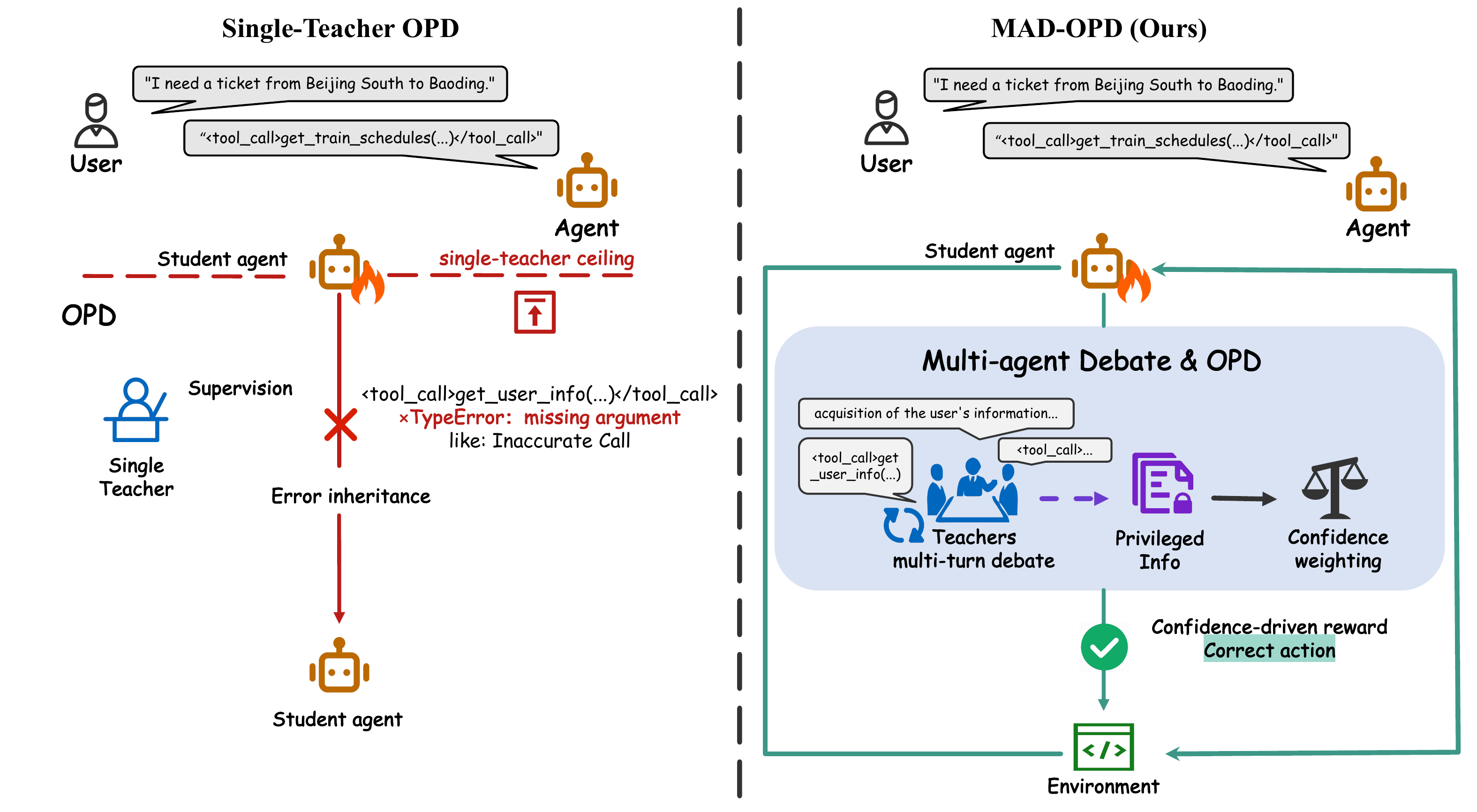}
  \caption{\textbf{Left:} a single teacher's erroneous tool call is faithfully inherited by the student (the \emph{single-teacher capability ceiling}). \textbf{Right:} \ours{}'s multi-round debate corrects each teacher's blind spots, producing supervision that outperforms individual teachers. A worked-out instance with the full debate trajectory is in App.~\ref{app:case_study}.}
  \label{fig:hook}
\end{figure}

On-policy distillation (OPD) has rapidly become a standard post-training recipe for modern LLMs~\citep{lu2025opd}, adopted in the strong-to-weak distillation of Qwen3~\citep{yang2025qwen3} and the multi-teacher OPD post-training of DeepSeek-V4~\citep{deepseekv4}; it offers dense on-policy supervision as an efficient alternative to outcome-reward reinforcement learning (RL)~\citep{shao2024deepseekmath} and off-policy sequence-level distillation~\citep{kim2016sequence}. Recent work extends OPD with privileged information~\citep{zhao2026opsd}, reward extrapolation~\citep{yang2026gopd}, verifiable rewards~\citep{yang2026rlsd}, and context internalization~\citep{ye2026opcd}; see~\citet{song2026survey} for a survey. However, three limitations remain. \textbf{(L1)~Single-teacher capability ceiling.} Existing OPD methods distill from a single teacher, upper-bounding the student by that teacher's capability; diagnostic work further shows that even a stronger teacher can fail to lift the student~\citep{li2026rethinkopd}. \textbf{(L2)~Narrow task coverage.} Current OPD research concentrates on math and commonsense reasoning~\citep{gu2024minillm,agarwal2024opd}, leaving \emph{agentic} domains~\citep{xi2023rise} (multi-step tool use with environment feedback) largely unexplored: per-step errors compound across long trajectories~\citep{li2026rethinkopd,fu2026revisitopd}. \textbf{(L3)~Ad-hoc divergence choice.} OPD is mathematically equivalent to dense token-level RL in which the divergence $D$ directly shapes the reward signal~\citep{yang2026gopd}; yet recent attempts to stabilize this objective~\citep{jang2026stable,jin2026eopd,fu2026revisitopd} each patch one failure mode without a principled rule linking $D$ to task structure, leaving divergence a task-agnostic hyperparameter.

Resolving these limitations requires looking beyond a single model. Multi-teacher distillation~\citep{chen2024magdi} and broader multi-agent systems~\citep{chen2025rag} show that collaboration among LLMs can solve problems no individual agent can; in particular, multi-agent debate (MAD)~\citep{du2023improving} produces an \emph{emergent collective intelligence} through iterative argumentation, and \citet{chen2024reconcile} demonstrates that confidence-weighted consensus among diverse LLMs can match or exceed the strongest individual model. Yet existing multi-teacher distillation remains \emph{off-policy}: the student passively consumes pre-computed teacher signals rather than benefiting from live teacher reactions to its own rollouts.

We propose \ours{} (\textbf{M}ulti-\textbf{A}gent \textbf{D}ebate-driven \textbf{O}n-\textbf{P}olicy \textbf{D}istillation; Figure~\ref{fig:hook}), the first framework to bring multi-agent debate into the OPD training loop as a source of debate-conditioned teacher distributions, transferring MAD's emergent collective intelligence to the student and breaking the single-teacher capability ceiling (\textbf{L1}). At every decision point, multiple teachers debate over the student's on-policy state, and the debate transcript provides token-level supervision with each teacher's contribution adaptively weighted by its post-debate confidence. To address \textbf{L2}, we propose \textit{On-Policy Agentic Distillation} (OPAD), a distillation method dedicated to agentic tasks that brings step-level environment interaction to OPD and is compatible with both single-teacher OPD and our \ours{}. To address \textbf{L3}, we cast divergence selection as reward design and derive a task-adaptive divergence principle: bounded JSD for agentic stability and reverse KL for coherent code generation, both verified theoretically and empirically (Sec.~\ref{sec:ablation}).

\vspace{2pt}
\noindent Our contributions are summarized as follows:
\begin{itemize}[leftmargin=1.5em,itemsep=1pt,topsep=2pt]
    \item \textbf{Task-adaptive divergence principle.} Through theoretical analysis we derive a task-adaptive principle that selects JSD for agentic OPD and reverse KL for code generation.
    \item \textbf{Multi-Agent Debate-driven OPD framework.} We introduce \ours{}, the first framework to use multi-agent debate as the source of token-level on-policy supervision, transferring MAD's emergent collective intelligence to the student and breaking the single-teacher capability ceiling.
    \item \textbf{OPAD for agentic tasks.} We propose \textit{On-Policy Agentic Distillation} (OPAD), which uses step-level sampling to stabilize OPD training under multi-step error compounding, extending OPD to an agentic regime rarely addressed by prior work.
\end{itemize}

\section{Related Work}
\label{sec:related}

\paragraph{On-Policy Distillation.}
Knowledge distillation~\citep{hinton2015distilling} transfers a teacher's soft predictions to a smaller student. \citet{agarwal2024opd, tan2023gkd} introduced OPD as a dense on-policy alternative to off-policy sequence-level distillation~\citep{kim2016sequence}, and OPD has since been adopted in frontier-model post-training~\citep{yang2025qwen3,deepseekv4} and surveyed by~\citet{song2026survey}. Recent diagnostic work investigates when OPD fails~\citep{li2026rethinkopd,fu2026revisitopd}, and \citet{jang2026stable} and \citet{jin2026eopd} propose objective-level reformulations to stabilize the reverse-KL objective. \citet{penaloza2026pi} extend OPD to leverage \emph{privileged information} that is visible only to a single teacher at training time. These methods all rely on a single teacher with task-agnostic divergence; \ours{} instead provides debate-generated privileged information and a task-adaptive divergence.

\paragraph{Multi-Agent Systems.}
Multi-agent systems (MAS) coordinate specialized agents to solve complex tasks; Chain-of-Agents~\citep{li2025coa} distills an entire MAS pipeline into a single foundation model. The MAS instantiation closest to our setting is multi-agent debate (MAD), where iterative argumentation among agents yields an \emph{emergent collective intelligence} that improves reasoning beyond any individual agent's capability~\citep{du2023improving,liang2024encouraging,chen2024reconcile,choi2025debate}. Existing MAD methods, however, consume debate outputs only at inference time; \ours{} brings debate inside the on-policy training loop as token-level supervision.

\paragraph{Multi-Teacher Distillation.}
Multi-teacher distillation pools complementary signals from several teachers. Classical recipes aggregate teacher outputs with adaptive per-input weighting~\citep{kim2016sequence,chen2023amtss}, while LLM-era variants distill multi-agent interaction graphs or debate traces into the student~\citep{chen2024magdi,wang2025smagdi,zhang2025mot,zhou2025dr}. These methods all lack inter-teacher interaction: teachers produce outputs independently and the student learns from a fixed aggregate. \ours{} adds on-policy inter-teacher debate as the supervision source.

\section{Preliminaries and Divergence Analysis}
\label{sec:prelim}

\subsection{On-Policy Distillation as Reward Design}
\label{sec:opd_rl}

\paragraph{OPD as dense token-level RL.}
Given a teacher $\pi^*$ and a student $\pi_\theta$ with parameters $\theta$, OPD minimizes a divergence along the student's own trajectories~\citep{agarwal2024opd, song2026survey}. For input $x$ from data distribution $\mathcal{D}$ and $\hat{y} \sim \pi_\theta(\cdot | x)$:
\begin{equation}
\label{eq:opd}
\mathcal{L}_{\textsc{opd}}(\theta) = \E_{x,\,\hat{y} \sim \pi_\theta} \left[ \frac{1}{|\hat{y}|} \sum_{t=1}^{|\hat{y}|} D\big(\pi^*(\cdot | x, c, \hat{y}_{<t}) \;\big\|\; \pi_\theta(\cdot | x, \hat{y}_{<t})\big) \right],
\end{equation}
where $D$ is a divergence and $c$ denotes \emph{privileged information} available only to the teacher~\citep{zhao2026opsd}; in our setting, $c$ is the multi-agent debate transcript. \citet{lu2025opd, yang2026gopd, hubotter2026sdpo} show that Eq.~\ref{eq:opd} corresponds to a dense token-level RL surrogate under teacher-forcing (App.~\ref{app:opd_grad}), where state $s_t = (x, \hat{y}_{<t})$, action $a_t = \hat{y}_t$, and per-token reward $r_D(s_t) \triangleq -D(\pi^*(\cdot|x,c,\hat{y}_{<t}) \| \pi_\theta(\cdot|x,\hat{y}_{<t}))$. Since this reward is entirely determined by $D$, \textbf{choosing a divergence is choosing the reward function}.

\paragraph{Notation.}
For brevity, when the position $t$ is clear from context we write $p \triangleq \pi^*(\cdot | x, c, \hat{y}_{<t})$ for the privileged teacher distribution and $q \triangleq \pi_\theta(\cdot | x, \hat{y}_{<t})$ for the student distribution over the vocabulary $\mathcal{V}$. The student's pre-softmax logit at coordinate $i$ is $z_i$, so $q(v) = \mathrm{softmax}(z)_v$. All logarithms are natural (base $e$).

\paragraph{The privileged $p$--$q$ gap.}
Privileged distillation creates a structural asymmetry: $p$ conditions on $c$ but $q$ does not, so $p$ and $q$ can disagree sharply at student-visited states. The asymmetry takes two forms: (a) the teacher assigns \emph{near-zero} probability to tokens the student samples ($p(v) \to 0$ while $q(v) > 0$), and (b) the teacher concentrates on \emph{multiple} valid tokens that the student must choose between. The first form dominates multi-step agentic tasks; the second dominates code generation.

\subsection{Theoretical Properties of Divergences}
\label{sec:divergence_theory}

Three divergences are standard candidates for $D$ in OPD:
\begin{align}
\text{Forward KL:} \quad D_{\kl}^{\rightarrow}(p \| q) &= \textstyle\sum_{v} p(v) \log \frac{p(v)}{q(v)}, \label{eq:fwd_kl} \\[2pt]
\text{Reverse KL:} \quad D_{\kl}^{\leftarrow}(p \| q) &\triangleq D_{\kl}(q \| p) = \textstyle\sum_{v} q(v) \log \frac{q(v)}{p(v)}, \label{eq:rev_kl} \\[2pt]
\text{JSD:} \quad \jsd_\beta(p \| q) &= \beta D_{\kl}(p \| m) + (1{-}\beta) D_{\kl}(q \| m), \quad m = \beta p + (1{-}\beta) q. \label{eq:jsd}
\end{align}
The two lemmas below characterize each divergence along two axes: \emph{geometry} (mode-covering vs.\ mode-seeking) and \emph{per-token logit-gradient boundedness}. We use $\beta = 0.5$ throughout: the canonical symmetric form attaining the maximum loss bound $\log 2$ (Lemma~\ref{lem:jsd_bound}.1), under which our gradient analysis is tightest.

\begin{lemma}[JSD: bounded loss and bounded logit gradient~\citep{lin1991divergence}]
\label{lem:jsd_bound}
For any distributions $p, q$ over $\mathcal{V}$ and $\beta \in (0,1)$:
\begin{enumerate}[leftmargin=1.5em,itemsep=0pt]
    \item \textbf{Loss bound:} $\jsd_\beta(p\|q) \in [0,\; H(\beta)]$, where $H(\beta) = -\beta\log\beta - (1{-}\beta)\log(1{-}\beta)$. For $\beta{=}0.5$: $\jsd_{0.5} \in [0,\; \log 2]$.
    \item \textbf{Logit-gradient bound:} for the symmetric case $\beta = 0.5$ used throughout this paper, the gradient with respect to student logits is uniformly bounded, $\|\nabla_z \jsd_{0.5}(p\|q)\|_\infty \leq 2$, \emph{regardless} of the support overlap between $p$ and $q$. The bound holds in the worst case with no independence or distributional assumption, and extends to any $\beta$ bounded away from $\{0, 1\}$ with a $\beta$-dependent constant (App.~\ref{app:proof_lem_jsd}).
\end{enumerate}
\end{lemma}
\textit{Proof in App.~\ref{app:proof_lem_jsd}.}

\begin{lemma}[Reverse KL: mode concentration~\citep{gu2024minillm}]
\label{lem:rev_kl_mode}
Let $p = \sum_{j=1}^J \alpha_j p_j$ be a mixture with $J$ modes having disjoint supports. Under the softmax parameterization, gradient-descent on $D_{\kl}(q\|p)$ converges to a stationary point $q^*$ that concentrates on the dominant mode $p_{j^*}$ ($\alpha_{j^*} = \max_j \alpha_j$), with cost $D_{\kl}(q^*\|p) = -\log\alpha_{j^*}$~\citep{gu2024minillm}.
\end{lemma}
\textit{Proof in App.~\ref{app:proof_lem_revkl}.}

\subsection{Task-Adaptive Divergence Principle}
\label{sec:divergence_principle}

We now translate the lemmas into task-specific propositions for the two regimes of \S\ref{sec:opd_rl}, which demand opposite divergences.

\paragraph{Agentic tasks: gradient stability under privileged $p$--$q$ gaps.}
\label{sec:jsd_theory}
A multi-step agentic OPD trajectory $\tau = (a_1, o_1, \ldots, a_M, o_M)$ accumulates errors over $M$ steps, where a mistake at step $m$ cascades to every subsequent decision. The privileged $p$--$q$ gap (\S\ref{sec:opd_rl}) takes its sharpest form here: under the dense OPD estimator (App.~\ref{app:opd_grad}), what propagates through training is not the trajectory \emph{return} but the per-step \emph{logit gradient} $\nabla_{z_t} D$.

\begin{proposition}[Per-token gradient stability under privileged $p$--$q$ gaps]
\label{prop:trajectory_var}
Under the dense OPD gradient estimator, the per-token logit gradient satisfies:
\begin{enumerate}[leftmargin=1.5em,itemsep=0pt]
    \item \textbf{JSD:} $\|\nabla_{z_t}\jsd_\beta(p\|q)\|_\infty \leq 2$ worst-case (Lemma~\ref{lem:jsd_bound}.2), with no support-overlap or independence assumption. By per-position decoupling, $\|\nabla_z \mathcal{L}_{\jsd}(\tau)\|_\infty \leq 2$ under either reduction (single-turn mean as in Eq.~\ref{eq:opd}, or sum-of-per-step-means as in OPAD), \emph{independent of trajectory length~$M$ and of the teacher--student gap}.
    \item \textbf{Reverse KL:} $\|\nabla_{z_t} D_{\kl}^{\leftarrow}(p\|q)\|_\infty$ is \emph{unbounded}: it contains a term proportional to $q(i)\log(q(i)/p(i))$, which diverges whenever $p(i) \to 0$ with $q(i) > 0$, the precise regime created by privileged supervision. Hence $\|\nabla_z \mathcal{L}_{\kl}^{\leftarrow}(\tau)\|_\infty$ is unbounded under \emph{any} reduction.
    \item \textbf{Forward KL:} $\|\nabla_{z_t} D_{\kl}^{\to}(p\|q)\|_\infty = \|q-p\|_\infty \leq 1$ is bounded; mode-covering nonetheless degrades agentic performance (App.~\ref{app:fwd_kl_instability}).
\end{enumerate}
\end{proposition}
\textit{Proof sketch.} Part~1 follows from Lemma~\ref{lem:jsd_bound}.2 plus per-position additive decomposition (the $\ell_\infty$ norm of the trajectory gradient inherits the per-position bound). Part~2 follows from the softmax-Jacobian term $q(i)\log(q(i)/p(i))$ diverging as $p(i)\!\to\!0$ with $q(i)\!>\!0$. Part~3 from $|q(i)-p(i)| \leq 1$. Full derivations in App.~\ref{app:proof_prop_traj}.

\paragraph{Code generation: mode concentration on coherent paths.}
\label{sec:revkl_theory}
For code, multiple correct implementations (different algorithms, naming, orderings) make the teacher distribution a mixture $p = \sum_{j=1}^J \alpha_j p_j$ over $J$ disjoint code paths.

\begin{proposition}[Reverse KL produces coherent code, forward KL and JSD do not]
\label{prop:rev_kl_code}
Let $p = \sum_{j=1}^J \alpha_j p_j$ be a per-token multi-modal code distribution with disjoint-support modes and a path-coherent dominant mode $j^*$ (consistent across positions). For a softmax-parameterized student $q^*$ trained against each divergence:
\begin{enumerate}[leftmargin=1.5em,itemsep=0pt]
    \item \textbf{Reverse KL:} $q^*$ concentrates on the dominant mode $p_{j^*}$ ($j^* = \arg\max_j \alpha_j$), assigning near-zero mass to other modes and avoiding incoherent splicing.
    \item \textbf{Forward KL:} $q^*$ covers all modes ($q^*(v) > 0$ on $\bigcup_j \mathrm{supp}(p_j)$), averaging incompatible code structures.
    \item \textbf{JSD:} $q^*$ partially concentrates on $p_{j^*}$ but retains non-negligible mass on secondary modes.
\end{enumerate}
\end{proposition}
\textit{Proof sketch.} Part~1 follows from Lemma~\ref{lem:rev_kl_mode}: the infinite penalty on inter-mode mass drives $q^*$ onto a single mode. Part~2 follows from $D_{\kl}^{\to}(p\|q) = \infty$ whenever $q(v) = 0$ on $\mathrm{supp}(p)$. Part~3 follows from JSD's mixture $m \geq \beta p$, which finitely penalizes inter-mode mass. Full derivations in App.~\ref{app:proof_prop_code}.

Combining Propositions~\ref{prop:trajectory_var} and~\ref{prop:rev_kl_code} yields the selection principle that determines which divergence each task structure demands.

\begin{remark}[Task-adaptive divergence selection principle]
\label{rem:task_dep}
For OPD with a privileged teacher, the divergence $D$ best matched to the task structure is selected as follows:
\begin{itemize}[leftmargin=1.5em,itemsep=2pt,topsep=2pt]
    \item \textbf{Multi-step agentic tasks} $\to$ $D = \jsd_\beta$. By Prop.~\ref{prop:trajectory_var}.1, the per-position logit gradient is bounded by $2$ uniformly across positions and privileged $p$--$q$ gaps, supporting stable long-trajectory training (App.~\ref{app:loss_stability}). Reverse KL admits no such finite bound under any reduction (Prop.~\ref{prop:trajectory_var}.2); forward KL is bounded but mode-covering produces incoherent rollouts (Prop.~\ref{prop:trajectory_var}.3 and App.~\ref{app:fwd_kl_instability}).
    \item \textbf{Code generation} $\to$ $D = D_{\kl}^{\leftarrow}$. By Prop.~\ref{prop:rev_kl_code}.1, the reverse-KL minimizer concentrates probability mass on a single coherent implementation path; forward KL and JSD provably splice incompatible modes (Prop.~\ref{prop:rev_kl_code}.2--3).
\end{itemize}
Sec.~\ref{sec:ablation} validates both predictions empirically (Table~\ref{tab:divergence}, Figure~\ref{fig:loss_stability}).
\end{remark}

\section{Method: \ours{}}
\label{sec:method}

\ours{} translates the task-adaptive divergence principle of Sec.~\ref{sec:divergence_principle} into a concrete training pipeline (Figure~\ref{fig:architecture}). Sec.~\ref{sec:mad} produces the privileged information $c$ that Eq.~\ref{eq:opd} requires via inter-teacher debate; Sec.~\ref{sec:weighting} aggregates teachers by post-debate confidence; Sec.~\ref{sec:loss} writes the resulting task-adaptive loss; Sec.~\ref{sec:opad} presents OPAD, an OPD-based training method for agentic tasks that brings step-level environment interaction.

\subsection{Multi-Agent Debate as Privileged Information Generation}
\label{sec:mad}

Given $K$ teachers $\{T_1, \ldots, T_K\}$ and the current state $s_m$ (the prompt $x$ for single-turn tasks, or the context $(x, \tau_{<m})$ for multi-step agentic tasks), debate runs for $R$ rounds. At round $1$ each teacher samples independently with no inter-teacher context; at every subsequent round each teacher revises its round-$r$ response $h_r^k$ after reading all prior responses from every teacher:
\begin{equation}
\label{eq:debate}
h_r^k \sim p_{T_k}\big(\cdot \mid s_m,\; \{h_{r'}^j\}_{j=1,\ldots,K,\; r'<r}\big), \quad k = 1, \ldots, K, \;\; r = 1, \ldots, R.
\end{equation}
The diverse initial responses (driven by differences in capacity or parameterization) seed the deliberation, and the iterative exchange lets teachers challenge each other and catch errors any individual teacher might miss. The full debate history $\mathcal{H}_m^R = \{h_r^k\}_{k,r}$ forms the privileged context:

\begin{definition}[Debate-generated privileged information]
\label{def:privileged}
The privileged context $c_m \triangleq \mathcal{H}_m^R$ is visible to teachers ($p_{T_k}(\cdot \mid s_m, c_m, \hat{y}_{<t})$) but hidden from the student ($p_S(\cdot \mid s_m, \hat{y}_{<t})$), instantiating Eq.~\ref{eq:opd} with collaboratively generated $c$.
\end{definition}

\begin{figure}[t]
  \centering
  \includegraphics[width=\linewidth]{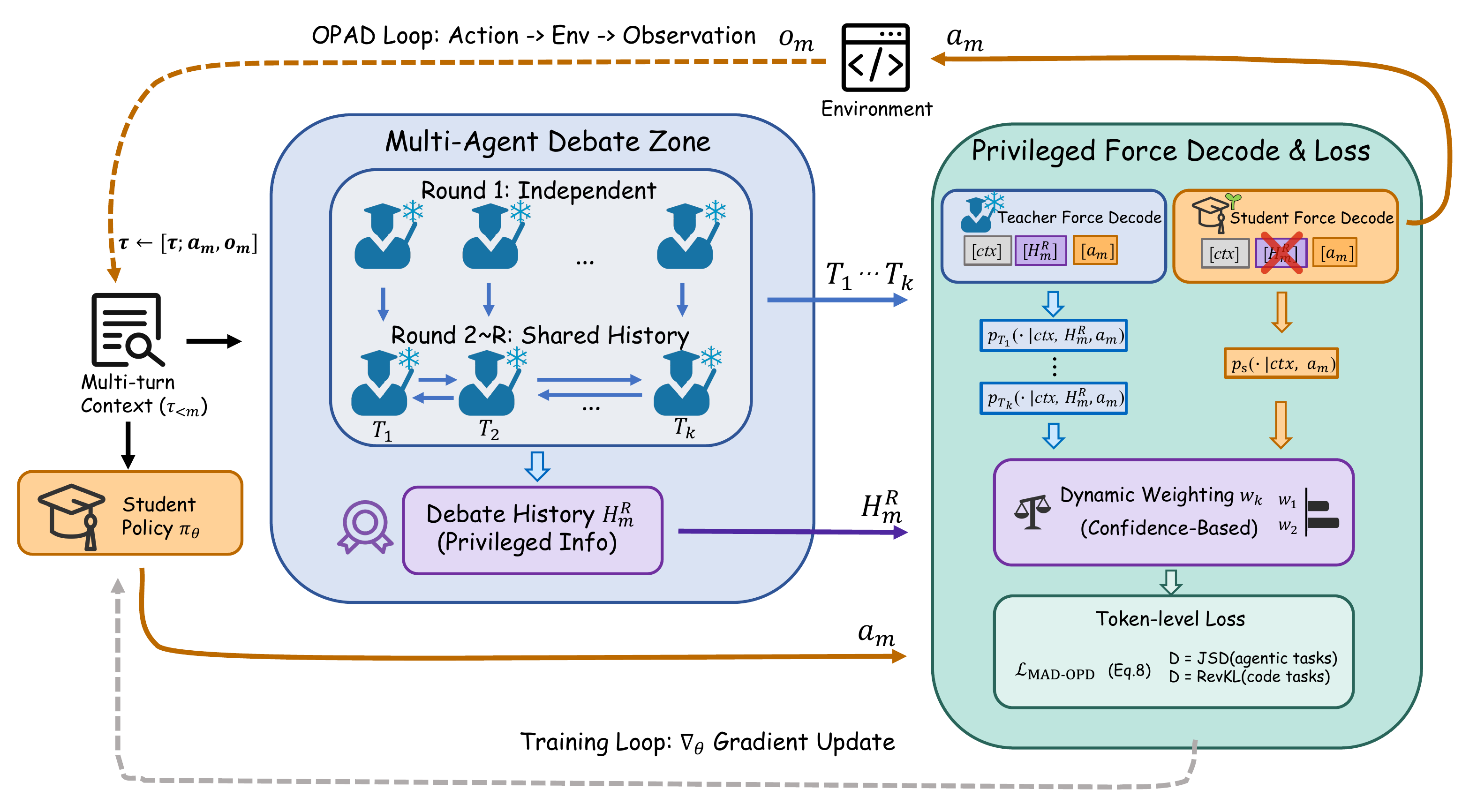}
  \caption{\textbf{The \ours{} Pipeline.} At trajectory step $m$, the student $\pi_\theta$ samples an on-policy action $a_m$; $K$ teachers debate for $R$ rounds to produce a transcript $\mathcal{H}_m^R$ visible only to the teachers, establishing the privileged $p$--$q$ gap. Teachers then force-decode $a_m$ and contribute to a confidence-weighted divergence $D$ ($\jsd$ for agents, $D_{\kl}^{\leftarrow}$ for code); gradients update only $\pi_\theta$. The dashed border marks the OPAD outer loop where the environment returns observation $o_m$ for the next step.}
  \label{fig:architecture}
\end{figure}

\subsection{Confidence-Based Dynamic Teacher Weighting}
\label{sec:weighting}

After round $R$, each teacher $T_k$ self-reports a scalar confidence score $c_k \in [0,100]$ extracted from its structured output, reflecting its certainty given the full debate history $\mathcal{H}_m^R$:
\begin{equation}\label{eq:confidence}
c_k = f_{\mathrm{conf}}\big(T_k;\; s_m,\, \mathcal{H}_m^R\big).
\end{equation}
Per-teacher weights are obtained by softmax normalization over the unit-scaled scores $\tilde{c}_k := c_k/100 \in [0,1]$:
\begin{equation}\label{eq:consensus_score}
w_k \;=\; \frac{\exp(\tilde{c}_k / \tau_{\mathrm{conf}})}{\textstyle\sum_{j=1}^{K} \exp(\tilde{c}_j / \tau_{\mathrm{conf}})}, \qquad \textstyle\sum_{k=1}^K w_k = 1,
\end{equation}
with temperature $\tau_{\mathrm{conf}} = 1.0$ throughout. Because $c_k$ is produced \emph{after} all $R$ rounds, the weights reflect post-deliberation certainty rather than pre-debate agreement: a teacher whose position weakens during debate contributes less. Robustness under biased or uninformative confidence is analyzed in App.~\ref{app:component_full}.

\subsection{Token-Level Distillation Objective}
\label{sec:loss}

The right panel of Figure~\ref{fig:architecture} shows the force-decode stage. Each teacher processes the student's on-policy sample $\hat{y}$ with the debate transcript $\mathcal{H}_m^R$ appended to its context: $p_{T_k}(\cdot \mid s_m, \mathcal{H}_m^R, \hat{y}_{<t})$. The student encodes the same tokens without this transcript: $p_S(\cdot \mid s_m, \hat{y}_{<t})$. The \ours{} loss aggregates per-teacher divergences weighted by their confidence scores:
\begin{equation}
\label{eq:mad_opd_loss}
\mathcal{L}_{\textsc{mad-opd}}(\theta) = \E_{s_m \sim \mathcal{D},\, \hat{y} \sim \pi_\theta(\cdot \mid s_m)} \!\left[ \frac{1}{|\hat{y}|} \sum_{t=1}^{|\hat{y}|} \sum_{k=1}^{K} w_k \cdot D\big(p_{T_k}(\cdot \mid s_m, \mathcal{H}_m^R, \hat{y}_{<t}) \;\big\|\; p_S(\cdot \mid s_m, \hat{y}_{<t})\big) \right].
\end{equation}
Following Remark~\ref{rem:task_dep}, we set $D = \jsd_\beta$ for agentic tasks and $D = D_{\kl}^{\leftarrow}$ for code generation, where $D_{\kl}^{\leftarrow}(P\|Q) \triangleq D_{\kl}(Q\|P)$ explicitly denotes the reverse KL. Divergences span the full vocabulary; teacher logits are treated as fixed targets, so gradients flow only through the student.

\subsection{On-Policy Agentic Distillation (OPAD)}
\label{sec:opad}

Guided by Prop.~\ref{prop:trajectory_var}, we propose \emph{On-Policy Agentic Distillation} (OPAD), instantiated here with \ours{} but compatible with single-teacher OPD. The student generates an on-policy trajectory $\tau = (a_1, o_1, \ldots, a_M, o_M)$ of length $M$: at step $m$, with state $s_m = (x, \tau_{<m})$, it samples action $a_m \sim \pi_\theta(\cdot \mid s_m)$ and the environment returns $o_m \sim P^{\mathrm{env}}(\cdot \mid s_m, a_m)$. At each step, teachers debate on $s_m$ and force-decode $a_m$ conditioned on the resulting transcript $\mathcal{H}_m^R$, giving the per-step OPAD loss:
\begin{equation}\label{eq:opad_reward}
\mathcal{L}_{D}^{\mathrm{opad}}(s_m) = \frac{1}{|a_m|}\sum_{t=1}^{|a_m|} \sum_{k=1}^{K} w_k \cdot D\big(p_{T_k}(\cdot \mid s_m, \mathcal{H}_m^R, a_{m,<t}) \;\big\|\; p_S(\cdot \mid s_m, a_{m,<t})\big).
\end{equation}
The total trajectory loss sums per-step contributions over the on-policy rollout:
\begin{equation}\label{eq:opad_total}
\mathcal{L}_{\mathrm{OPAD}}(\theta) = \E_{x \sim \mathcal{D},\, \tau \sim \pi_\theta(\cdot \mid x)}\!\left[\sum_{m=1}^{M} \mathcal{L}_{D}^{\mathrm{opad}}(s_m)\right].
\end{equation}
The key adaptive property: because debate happens \emph{per step} and conditions on actual observations $\{o_1, \ldots, o_{m-1}\}$ (not hypothetical ones), supervision adapts to whatever trajectory the student produces. Algorithm~\ref{alg:mad_opd} (App.~\ref{app:algorithm}) gives the complete procedure.

\section{Experiments}
\label{sec:exp}

We evaluate \ours{} across six teacher--student configurations, two model families, and five benchmarks. Four research questions guide the analysis: \textbf{(RQ1)}~Does debate break the single-teacher capability ceiling? \textbf{(RQ2)}~How do gains scale across teacher--student capability? \textbf{(RQ3)}~Do debate and confidence weighting yield strictly non-redundant gains over naive on-policy aggregation? \textbf{(RQ4)}~Do divergences behave as the theory of Sec.~\ref{sec:divergence_principle} predicts?

\subsection{Experimental Setup}
\label{sec:setup}

We evaluate \ours{} on two model families, Qwen3~\citep{yang2025qwen3} and Qwen3.5, across six teacher--student configurations spanning teacher-to-student size ratios from 4.7$\times$ to 15.5$\times$: two Qwen3 teacher pairs (14B+8B and 32B+30B-A3B, with the 30B-A3B teacher being the \texttt{Qwen3-30B-A3B-Instruct-2507} release) and one Qwen3.5 pair (27B+9B), with students from 1.7B to 14B. Training data is ToolACE~\citep{liu2024toolace} ($\sim$16K step-split agentic instances) for agentic tasks and OpenThoughts3~\citep{guha2025openthoughts} (30K problems) for code tasks. The benchmarks instantiate the agentic-vs-code split of our task-adaptive divergence principle (Sec.~\ref{sec:divergence_principle}), namely multi-step agentic tool use (BFCL-v4~\citep{yan2024bfcl}, $\tau^2$-Bench~\citep{barres2025tau2bench}, VitaBench~\citep{he2025vitabench}) and single-turn code generation (LiveCodeBench v6~\citep{jain2024livecodebench}, MBPP+~\citep{austin2021mbpp,liu2023evalplus}); we exclude math (already covered by recent OPD work~\citep{li2026rethinkopd,fu2026revisitopd,yang2026rlsd}) and focus on the less-explored agentic and code regimes that expose the divergence-selection question motivating our theory. Against this setup we benchmark four reference paradigms: the undistilled instruct model (\textbf{Base}); single-teacher OPD with the stronger teacher of each pair~\citep{agarwal2024opd} (\textbf{OPD}); naive multi-teacher OPD without debate (\textbf{MT-OPD}, equal weights); and off-policy multi-teacher sequence distillation with a 7:3 strong-to-weak mix (\textbf{MT-SeqKD}~\citep{kim2016sequence}); \ours{} is orthogonal to RL-based methods such as GRPO~\citep{shao2024deepseekmath}. All models run in non-thinking mode; hyperparameters ($K{=}2$, $R{=}2$, $\beta{=}0.5$), debate prompts, and full protocols are in App.~\ref{app:details}.

\subsection{Main Results}
\label{sec:main_results}

\begin{table}[!t]
\caption{\textbf{Main results (\%).} Agentic tasks use $D{=}\jsd_\beta$; code tasks use $D{=}D_{\kl}^{\leftarrow}$ (per Remark~\ref{rem:task_dep}). \textbf{Ag-Avg} / \textbf{Co-Avg}: per-task-group means. Agentic scores are accuracy; code scores are pass@1 averaged over 16 random seeds (subscript: std). \textbf{Bold}: best; \underline{underline}: second best per config.}
\label{tab:main}
\centering
\renewcommand{\arraystretch}{0.85}
\setlength{\tabcolsep}{3.5pt}
\scriptsize
\begin{tabular*}{\textwidth}{@{\extracolsep{\fill}} l ccc c cc c c @{}}
\toprule
& \multicolumn{4}{c}{\textbf{Agentic Tasks} ($D{=}\jsd$)} & \multicolumn{3}{c}{\textbf{Code Tasks} ($D{=}D_{\kl}^{\leftarrow}$)} & \\
\cmidrule(lr){2-5} \cmidrule(lr){6-8}
\textbf{Method} & BFCL-v4 & $\tau^2$ & Vita & Ag-Avg & LCB-v6 & MBPP+ & Co-Avg & \textbf{Avg} \\
\midrule
\multicolumn{9}{l}{\cellcolor{gray!8}\textit{Qwen3 \; 14B{+}8B $\to$ 1.7B}} \\
\cmidrule{1-9}
Base & 19.65 & 13.47 & 7.81 & 13.64 & 15.43{\tiny$_{\pm 0.61}$} & 51.27{\tiny$_{\pm 0.63}$} & 33.35 & 21.53 \\
MT-SeqKD & 15.52 & 13.02 & 9.06 & 12.53 & 15.82{\tiny$_{\pm 0.74}$} & 51.22{\tiny$_{\pm 0.75}$} & 33.52 & 20.93 \\
OPD & 20.91 & \underline{18.72} & \underline{12.16} & \underline{17.26} & \underline{19.26}{\tiny$_{\pm 0.84}$} & \underline{54.28}{\tiny$_{\pm 0.78}$} & \underline{36.77} & \underline{25.07} \\
MT-OPD & \underline{20.96} & 18.07 & 11.31 & 16.78 & 19.07{\tiny$_{\pm 0.73}$} & 51.85{\tiny$_{\pm 1.14}$} & 35.46 & 24.25 \\
\textbf{\ours{}} & \textbf{23.14} & \textbf{20.59} & \textbf{13.28} & \textbf{19.00} & \textbf{21.39}{\tiny$_{\pm 1.21}$} & \textbf{55.92}{\tiny$_{\pm 1.03}$} & \textbf{38.66} & \textbf{26.86} \\
\midrule
\multicolumn{9}{l}{\cellcolor{gray!8}\textit{Qwen3 \; 14B{+}8B $\to$ 4B}} \\
\cmidrule{1-9}
Base & 23.34 & 18.00 & 14.60 & 18.65 & 22.73{\tiny$_{\pm 0.65}$} & 60.27{\tiny$_{\pm 0.35}$} & 41.50 & 27.79 \\
MT-SeqKD & 22.51 & 26.07 & 15.85 & 21.48 & 23.18{\tiny$_{\pm 0.21}$} & 58.36{\tiny$_{\pm 1.32}$} & 40.77 & 29.19 \\
OPD & 24.79 & \underline{29.17} & 15.81 & 23.26 & \underline{25.00}{\tiny$_{\pm 0.52}$} & \underline{63.82}{\tiny$_{\pm 1.21}$} & \underline{44.41} & \underline{31.72} \\
MT-OPD & \underline{25.65} & 29.06 & \underline{17.72} & \underline{24.14} & 23.29{\tiny$_{\pm 0.41}$} & 60.64{\tiny$_{\pm 0.72}$} & 41.97 & 31.27 \\
\textbf{\ours{}} & \textbf{27.08} & \textbf{31.46} & \textbf{18.53} & \textbf{25.69} & \textbf{29.83}{\tiny$_{\pm 0.56}$} & \textbf{66.42}{\tiny$_{\pm 0.92}$} & \textbf{48.12} & \textbf{34.66} \\
\midrule
\multicolumn{9}{l}{\cellcolor{gray!8}\textit{Qwen3 \; 32B{+}30B\text{-}A3B $\to$ 4B}} \\
\cmidrule{1-9}
Base & 23.34 & 18.00 & 14.60 & 18.65 & 22.73{\tiny$_{\pm 0.65}$} & 60.27{\tiny$_{\pm 0.35}$} & 41.50 & 27.79 \\
MT-SeqKD & 23.09 & \textbf{35.60} & 13.87 & 24.19 & 25.06{\tiny$_{\pm 0.26}$} & 63.29{\tiny$_{\pm 1.15}$} & 44.17 & 32.18 \\
OPD & \underline{27.70} & 32.39 & 16.20 & 25.43 & 26.96{\tiny$_{\pm 0.41}$} & \underline{64.43}{\tiny$_{\pm 1.03}$} & \underline{45.70} & \underline{33.54} \\
MT-OPD & 24.95 & 34.60 & \underline{17.29} & \underline{25.61} & \underline{27.14}{\tiny$_{\pm 0.43}$} & 61.92{\tiny$_{\pm 0.85}$} & 44.53 & 33.18 \\
\textbf{\ours{}} & \textbf{27.76} & \underline{34.62} & \textbf{19.53} & \textbf{27.30} & \textbf{30.42}{\tiny$_{\pm 0.43}$} & \textbf{66.37}{\tiny$_{\pm 0.95}$} & \textbf{48.40} & \textbf{35.74} \\
\midrule
\multicolumn{9}{l}{\cellcolor{gray!8}\textit{Qwen3 \; 32B{+}30B\text{-}A3B $\to$ 8B}} \\
\cmidrule{1-9}
Base & 28.94 & 30.08 & 17.16 & 25.39 & 20.93{\tiny$_{\pm 0.76}$} & 63.56{\tiny$_{\pm 0.61}$} & 42.25 & 32.13 \\
MT-SeqKD & \underline{30.62} & \underline{39.06} & 16.78 & 28.82 & 21.93{\tiny$_{\pm 0.31}$} & 66.36{\tiny$_{\pm 1.52}$} & 44.14 & 34.95 \\
OPD & 29.94 & 36.91 & 22.02 & 29.62 & \underline{26.83}{\tiny$_{\pm 0.34}$} & 66.64{\tiny$_{\pm 0.93}$} & 46.73 & 36.47 \\
MT-OPD & 29.96 & 38.73 & \underline{22.47} & \underline{30.39} & 26.57{\tiny$_{\pm 0.38}$} & \underline{67.20}{\tiny$_{\pm 1.31}$} & \underline{46.89} & \underline{36.99} \\
\textbf{\ours{}} & \textbf{31.39} & \textbf{40.21} & \textbf{23.11} & \textbf{31.57} & \textbf{30.71}{\tiny$_{\pm 0.74}$} & \textbf{67.32}{\tiny$_{\pm 1.12}$} & \textbf{49.02} & \textbf{38.55} \\
\midrule
\multicolumn{9}{l}{\cellcolor{gray!8}\textit{Qwen3 \; 32B{+}30B\text{-}A3B $\to$ 14B}} \\
\cmidrule{1-9}
Base & 32.46 & 33.70 & 21.70 & 29.29 & 25.57{\tiny$_{\pm 0.87}$} & 75.13{\tiny$_{\pm 0.78}$} & 50.35 & 37.71 \\
MT-SeqKD & \underline{33.99} & 40.26 & 23.24 & 32.50 & 27.25{\tiny$_{\pm 1.01}$} & 75.36{\tiny$_{\pm 0.85}$} & 51.30 & 40.02 \\
OPD & 32.16 & 40.53 & 24.47 & 32.39 & 28.73{\tiny$_{\pm 0.84}$} & \underline{76.73}{\tiny$_{\pm 1.04}$} & \underline{52.73} & 40.52 \\
MT-OPD & 30.52 & \underline{42.42} & \textbf{25.44} & \underline{32.79} & \underline{29.17}{\tiny$_{\pm 0.84}$} & 75.50{\tiny$_{\pm 0.62}$} & 52.34 & \underline{40.61} \\
\textbf{\ours{}} & \textbf{34.35} & \textbf{43.59} & \underline{24.84} & \textbf{34.26} & \textbf{31.36}{\tiny$_{\pm 1.02}$} & \textbf{77.42}{\tiny$_{\pm 1.31}$} & \textbf{54.39} & \textbf{42.31} \\
\midrule
\multicolumn{9}{l}{\cellcolor{gray!8}\textit{Qwen3.5 \; 27B{+}9B $\to$ 4B}} \\
\cmidrule{1-9}
Base & \underline{35.76} & 42.25 & 9.73 & \underline{29.25} & 21.93{\tiny$_{\pm 1.02}$} & 62.73{\tiny$_{\pm 0.98}$} & 42.33 & 34.48 \\
MT-SeqKD & 24.04 & 35.54 & 7.33 & 22.30 & \underline{29.06}{\tiny$_{\pm 0.62}$} & \underline{66.73}{\tiny$_{\pm 0.97}$} & \underline{47.90} & 32.54 \\
OPD & 34.76 & \underline{43.39} & 9.07 & 29.07 & 28.32{\tiny$_{\pm 1.02}$} & 65.21{\tiny$_{\pm 0.57}$} & 46.77 & \underline{36.15} \\
MT-OPD & 33.98 & 39.53 & \underline{10.49} & 28.00 & 28.12{\tiny$_{\pm 0.42}$} & 63.49{\tiny$_{\pm 0.94}$} & 45.80 & 35.12 \\
\textbf{\ours{}} & \textbf{36.72} & \textbf{44.83} & \textbf{10.83} & \textbf{30.79} & \textbf{30.26}{\tiny$_{\pm 0.46}$} & \textbf{67.32}{\tiny$_{\pm 1.15}$} & \textbf{48.79} & \textbf{37.99} \\
\bottomrule
\end{tabular*}
\end{table}

\noindent\textbf{(RQ1) Takeaway 1: \ours{} consistently outperforms the strongest single-teacher OPD across all six configurations.} Table~\ref{tab:main} shows \ours{} ranking first on overall Avg in every configuration, while each baseline fails on a structurally distinct axis: single-teacher OPD is capped by its teacher; MT-OPD's per-token gradient conflict drops it below single-teacher OPD on code, because per-token averaging of two teacher distributions interpolates between distinct coherent code paths each teacher prefers and produces incoherent supervision the student then absorbs; off-policy MT-SeqKD inherits the exposure bias on-policy training avoids. \ours{} bypasses all three by turning inter-teacher disagreement into a debate transcript that establishes consensus \emph{before} force-decoding, so the resulting supervision is densest where the privileged $p$--$q$ gap is widest, providing a sharper signal than single-teacher argmax decoding (Prop.~\ref{prop:trajectory_var}). As a concrete instance, App.~\ref{app:lcb_surpass} reports a 4B student trained with our 14B+8B teacher debate outperforming the 14B teacher on LCB-v6 by $+4.26\%$ pass@1 and $+10.29\%$ Best-of-$N$ (BoN@16), at competitive token cost; full BoN@16 results across all six configurations and on MBPP+ are in App.~\ref{app:bon}.

\subsection{Scaling Behavior}
\label{sec:scaling}

\noindent\textbf{(RQ2) Takeaway 2: \ours{} generalizes across families and configurations, with gains scaling alongside teacher and student capability.} Figure~\ref{fig:scaling} probes scaling along three axes: teacher capability, student capability, and cross-family transfer.


\begin{figure}[t]
\vspace{-8pt}
\centering
\begin{minipage}[t]{0.49\linewidth}
  \centering
  \includegraphics[width=\linewidth]{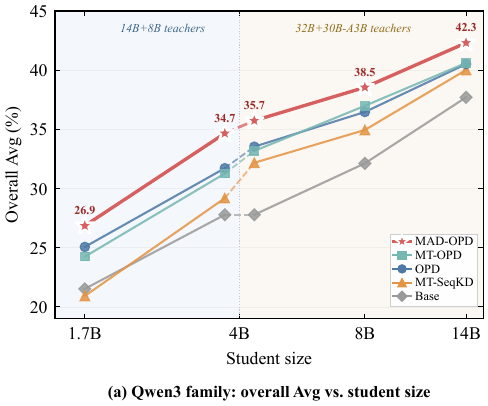}
\end{minipage}\hspace{1pt}%
\begin{minipage}[t]{0.49\linewidth}
  \centering
  \includegraphics[width=\linewidth]{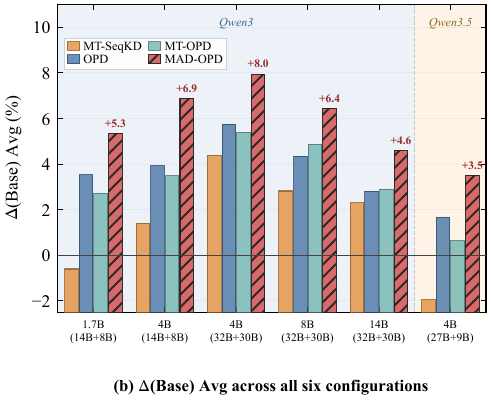}
\end{minipage}
\caption{\textbf{Scaling behavior.} \textbf{(a)}~Overall Avg (\%) vs.\ student size on the Qwen3 family; solid lines connect within-pair points, dashed segments at 4B connect cross-pair points. \textbf{(b)}~$\Delta$(Base) Avg (\%), the gain over the un-distilled Base, across all six configurations.}
\label{fig:scaling}
\end{figure}

\textbf{Debate amplifies multi-teacher complementarity.} Naive MT-OPD already exceeds single-teacher OPD on Ag-Avg in 4 of 6 configurations, showing that adding a weaker teacher's signal helps even without debate; \ours{} extends this lead to all 6 configurations, so the gain reflects \emph{debated} complementarity rather than raw ensemble capacity.

\textbf{Student--teacher capability gap shapes the scaling behavior.} Reading $\Delta$(Base) off Figure~\ref{fig:scaling}(b): an over-large teacher--student gap caps the gain ($\Delta$ rises from $+5.3\%$ at 1.7B to $+6.9\%$ at 4B under the 14B+8B teachers, since at $\sim$13$\times$ ratio the student cannot sample on-policy trajectories inside the teachers' competence regime~\citep{li2026rethinkopd}); once the student is capable enough, a stronger teacher pool injects additional usable signal (at 4B, switching from 14B+8B to 32B+30B-A3B teachers lifts \ours{} from $34.7\%$ to $35.7\%$, dashed segment in Figure~\ref{fig:scaling}(a)); and on Qwen3.5 27B+9B$\to$4B, where baselines lift Base by at most $+1.67\%$, \ours{}'s $+3.51\%$ comes from converting per-query inter-teacher disagreement into a confidence-weighted consensus that uniform multi-teacher aggregation cannot recover (MT-OPD$<$OPD on this configuration).

\vspace{-3pt}
\subsection{Component and Divergence Analysis}
\label{sec:ablation}
\vspace{-2pt}

\noindent\textbf{(RQ3) Takeaway 3: \ours{}'s components are non-redundant.} Figure~\ref{fig:ablation_compact}(a) (full numbers in App.~\ref{app:component_full}) shows that removing any of multi-teacher aggregation, on-policy training, debate, or confidence weighting hurts overall scores, confirming each contributes independently. The default $R{=}2$ (one inter-teacher revision round on top of independent generation) is the empirical optimum; $R{=}3$ overshoots due to prompt-context bloat (full $R\in\{0,1,2,3\}$ sweep in App.~\ref{app:rounds_ablation}).

\noindent\textbf{(RQ4) Takeaway 4: Divergence selection is dictated by task geometry, with JSD winning on agentic tasks and reverse KL on code.} A controlled ablation fixing all hyperparameters except $D$ (full numbers in App.~\ref{app:divergence_full}) shows the dominance pattern \emph{inverts} along the agentic/code split: JSD leads on agentic tasks while reverse KL leads on code, with forward KL trailing on both, as Prop.~\ref{prop:trajectory_var}.3 and Prop.~\ref{prop:rev_kl_code}.2 jointly predict (mode-covering harms both regimes; mode concentration matches code's single-coherent-path structure). Training-dynamics evidence in Figure~\ref{fig:loss_stability} corroborates this (App.~\ref{app:loss_stability}).

\begin{figure}[t]
\centering
\includegraphics[width=0.96\linewidth]{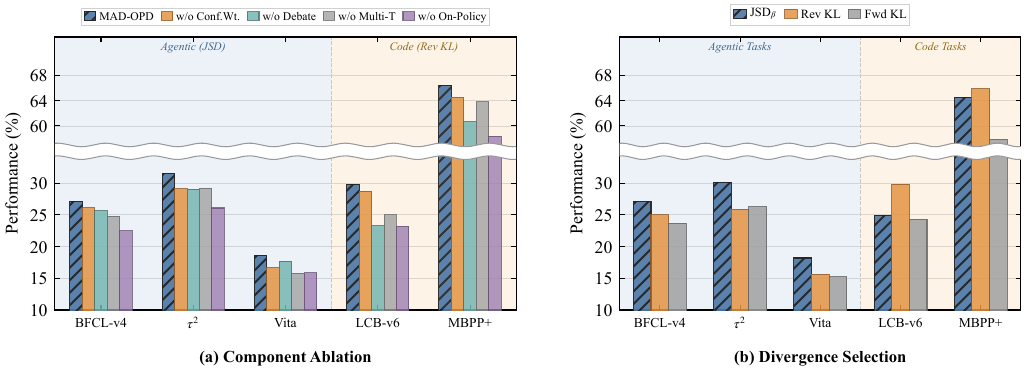}
\caption{\textbf{Ablation study} (14B{+}8B$\to$4B). Shading separates agentic/code benchmarks. \textbf{(a)}~Component ablation: debate adds $+4.6\%$ Co-Avg vs.\ MT-OPD; confidence weighting adds $+1.7\%$ Ag-Avg. \textbf{(b)}~Per-divergence ablation (full numbers: Table~\ref{tab:divergence}).}
\label{fig:ablation_compact}
\end{figure}

\vspace{-8pt}
\section{Discussion and Conclusion}
\label{sec:conclusion}
\vspace{-5pt}

We presented \ours{}, a multi-teacher debate framework for on-policy distillation paired with OPAD for stable multi-step agentic training, and a task-adaptive divergence principle (JSD for agentic OPD, reverse KL for code). Across six teacher--student configurations and five benchmarks, \ours{} ranks first in every configuration; a 4B student trained under the 14B+8B teacher debate even exceeds its 14B teacher on LCB-v6 ($+4.26\%$ pass@1, $+10.29\%$ BoN@16) at competitive token cost. This 4B-beats-14B case suggests the bottleneck of on-policy distillation is teacher-pool \emph{diversity} rather than single-teacher capability, pointing to a complementary path to chasing larger single teachers. Future work includes richer debate structures, cross-vocabulary distillation, and extending OPAD to longer-horizon agentic environments where per-step error compounding becomes more severe.

\vspace{-4pt}\paragraph{Limitations.}
\ours{} requires teacher token-level distributions, excluding black-box (API-only) models, and assumes a shared vocabulary between teacher and student; black-box and cross-vocabulary distillation are open extensions. Training cost also scales with $K \times R$ teacher forward passes per step.

\bibliographystyle{plainnat}
\bibliography{references}

\appendix

\section{Algorithm Pseudocode}
\label{app:algorithm}

Algorithm~\ref{alg:mad_opd} provides the complete \ours{} training procedure, covering both single-step (code generation, $M{=}1$) and multi-step OPAD (agentic tasks, $M{>}1$) settings. The procedure follows a \emph{think-then-critique} paradigm: teachers first deliberate on the current state $s_m$ to produce the privileged consensus $\mathcal{H}_m^R$, and only then evaluate the student's already-sampled action $a_m$ via privileged force-decoding. Teachers therefore argue about the high-level decision conditioned on $s_m$, not about specific tokens of the student's draft.

\begin{algorithm}[H]
\caption{\ours{}: Multi-Agent Debate-Driven On-Policy Distillation}
\label{alg:mad_opd}
\small
\begin{algorithmic}[1]
\REQUIRE Teachers $\{T_1, \ldots, T_K\}$, student $\pi_\theta$, data $\mathcal{D}$, debate rounds $R$, divergence $D$, total steps $M$, learning rate $\eta$
\REPEAT
    \STATE Sample batch $\mathcal{B}$ from $\mathcal{D}$
    \FOR{each $x \in \mathcal{B}$}
        \STATE $\mathcal{L}_x \leftarrow 0$;\; $\tau \leftarrow \emptyset$
        \FOR{step $m = 1, \ldots, M$}
            \STATE $a_m \sim \pi_\theta(\cdot \mid x, \tau)$ \COMMENT{on-policy; $M{=}1$ for non-agentic tasks}
            \FOR{round $r = 1, \ldots, R$; \; teacher $k = 1, \ldots, K$ \textbf{in parallel}}
                \STATE $h_r^k \sim p_{T_k}(\cdot \mid x, \tau, \{h_{r'}^j\}_{j,r'<r})$
            \ENDFOR
            \STATE $\mathcal{H}_m^R \leftarrow \{h_r^k\}_{k,r}$ \COMMENT{privileged context}
            \STATE $c_k \leftarrow f_{\mathrm{conf}}(T_k;\; x, \tau, \mathcal{H}_m^R)$;\; $w_k \leftarrow \mathrm{softmax}_k\!\big(c_{\cdot}/(100\,\tau_{\mathrm{conf}})\big)$ \;\; $\forall k$ \COMMENT{confidence weights}
            \STATE $\mathbf{p}_{T_k}^{(t)} \leftarrow p_{T_k}(\cdot \mid x, \tau, \mathcal{H}_m^R, a_{m,<t})$ \;\; $\forall k, t$ \COMMENT{force-decode \textbf{with} $\mathcal{H}_m^R$}
            \STATE $\mathbf{p}_S^{(t)} \leftarrow \pi_\theta(\cdot \mid x, \tau, a_{m,<t})$ \;\; $\forall t$ \COMMENT{student \textbf{without} $\mathcal{H}_m^R$}
            \STATE $\mathcal{L}_x \leftarrow \mathcal{L}_x + \frac{1}{|a_m|}\textstyle\sum_{t} \sum_{k} w_k \cdot D\big(\mathbf{p}_{T_k}^{(t)} \,\big\|\, \mathbf{p}_S^{(t)}\big)$
            \STATE $o_m \sim P^{\mathrm{env}}(\cdot \mid x, \tau, a_m)$;\; $\tau \leftarrow \tau \oplus (a_m, o_m)$ \COMMENT{env step}
        \ENDFOR
    \ENDFOR
    \STATE $\theta \leftarrow \theta - \eta \nabla_\theta \frac{1}{|\mathcal{B}|} \textstyle\sum_{x} \mathcal{L}_x$ \COMMENT{update $\pi_\theta$ only}
\UNTIL{convergence}
\end{algorithmic}
\end{algorithm}

\section{Experimental Details}
\label{app:details}

This appendix describes our complete training and evaluation setup. Sections~\ref{app:training_config}--\ref{app:hyperparams} cover the training pipeline; Sections~\ref{app:bfcl_detail}--\ref{app:mbpp_detail} provide per-benchmark evaluation details for the five benchmarks used in the main paper.

\subsection{Training Configuration}
\label{app:training_config}
All models operate in non-thinking mode (\texttt{enable\_thinking=False}) throughout training and evaluation. Optimization uses AdamW~\citep{loshchilov2019adamw} with $\beta_1{=}0.9$, $\beta_2{=}0.999$, weight decay $0.01$, learning rate $1{\times}10^{-5}$ with cosine decay and 5\% warmup. Effective batch size is 128 with gradient clipping at norm $1.0$. Agentic tasks use a maximum sequence length of 4096 tokens; code tasks allow up to 16384 tokens per generation. All configurations train for 1 epoch with checkpoints saved every 10 steps; the best checkpoint is selected by overall benchmark performance. The 30B-A3B teacher is the \texttt{Qwen3-30B-A3B-Instruct-2507} release; all other Qwen3/Qwen3.5 teachers and students use the standard Instruct checkpoints of the corresponding size.

\subsection{Infrastructure}
\label{app:infra}
All experiments run on 8$\times$ NVIDIA H20 GPUs with DeepSpeed ZeRO-3 for distributed training. Teacher models are served via two separate single-GPU processes: (i)~\textbf{vLLM} generates debate-round text responses, and (ii)~a \textbf{sidecar process} extracts token-level logit distributions for the distillation loss. The sidecar exposes the full vocabulary distribution rather than top-$k$ truncation, so that JSD and reverse-KL gradients see the complete teacher density. End-to-end runtime: a single 4B-student agentic-task run completes in approximately 16 hours; code-task training takes approximately 32 hours due to the longer maximum sequence length.

\subsection{Training Data}
\label{app:train_data}
\textbf{Agentic data.} Training uses the ToolACE training set~\citep{liu2024toolace}, decomposed into per-step instances under the OPAD protocol: each multi-turn trajectory is split into independent step-level training samples, producing ${\sim}$16K examples in total. At each step the student rolls out an action conditioned on the prior dialogue and tool-call history; teachers debate on the same step state and force-decode the student-sampled action. \textbf{Code data.} We sample 30K problems from OpenThoughts3~\citep{guha2025openthoughts}. Each training example is a single on-policy completion: the student generates a code rollout, both teachers produce token-level distributions conditioned on the student's prefix and (after debate) the debate history, and the consensus distribution serves as the training signal. Agentic and code tasks train separate checkpoints with no shared data; this matches the task-adaptive divergence design of Sec.~\ref{sec:divergence_principle}.

\subsection{Baseline Setup}
\label{app:baselines}
\textbf{Base.} The undistilled instruct checkpoint of the corresponding student size, evaluated directly. \textbf{OPD}~\citep{agarwal2024opd}. Single-teacher on-policy distillation using the stronger teacher of each pair (e.g., Qwen3-14B for the 14B+8B pair); divergence and decoding match \ours{}. \textbf{MT-OPD}. Naive multi-teacher OPD: both teachers produce token-level distributions independently and the loss aggregates them with uniform weights ($w_k=1/K$); no inter-teacher debate. \textbf{MT-SeqKD}~\citep{kim2016sequence}. Off-policy multi-teacher sequence distillation routed by a 7:3 strong-to-weak teacher mix: the strong teacher generates 70\% of the training sequences and the weak teacher 30\%, mixed at the sequence level. The student is trained with sequence-level cross-entropy on the routed teacher outputs.

\subsection{Debate Configuration}
\label{app:debate_config}
Each debate runs $R{=}2$ rounds with $K{=}2$ teachers. In round 1, both teachers see the task context and the student's generated prefix and respond independently with no inter-teacher context; in round 2, each teacher additionally sees the round-1 responses (including confidence scores) prefixed with speaker identity and revises its answer. The debate temperature is 0.7 to encourage diverse initial proposals while promoting convergence. After the final round, each teacher provides a confidence score (integer in $[0, 100]$); each score is divided by $100$ to yield $\tilde{c}_k \in [0,1]$ and then softmax-normalized with temperature $\tau_{\mathrm{conf}}{=}1.0$ to produce the per-teacher weights (Eq.~\ref{eq:consensus_score}). The debate prompt is kept identical across all configurations and model families. App.~\ref{app:rounds_ablation} ablates $R \in \{0, 1, 2, 3\}$ and confirms $R{=}2$ as the optimal operating point.

\subsection{Evaluation Protocol}
\label{app:eval_protocol}
All methods are trained for 1 epoch. In large-scale LLM post-training, exhaustive checkpoint selection on fully separate validation suites is often prohibitively expensive, especially for agentic evaluation (which requires user-simulator API calls) and code evaluation (which requires per-checkpoint sandbox runs). We therefore adopt a practical checkpoint-selection protocol: checkpoints are evaluated every 10 steps for agentic tasks on \textbf{BFCL-v4} and every 5 steps for code tasks on \textbf{LCB-v6}, and the highest-scoring checkpoint for each task type is then evaluated once on all five benchmarks for the final reported numbers. The same protocol is applied to all methods. Under this protocol, BFCL-v4 and LCB-v6 serve as development benchmarks for checkpoint selection, while $\tau^2$-Bench, VitaBench, and MBPP+ are held out from model selection; this is consistent with common practice in prior OPD work~\citep{agarwal2024opd,song2026survey}.

Code-task pass@1 scores are averaged over 16 decoding seeds (subscript std in Tables~\ref{tab:main} and~\ref{tab:bon_results}); agentic-task scores are reported as a single run per the deterministic-simulator setup of each agentic benchmark, where the reported score is itself a mean over many internal subtasks or domains (8 subtasks for BFCL-v4, 3 domains $\times$ many tasks each for $\tau^2$-Bench and VitaBench), substantially reducing single-run variance. Following the Qwen3 technical report~\citep{yang2025qwen3} (recommended decoding parameters for non-thinking mode), we use sampling with temperature $0.7$, top-$p=0.8$, and presence penalty $1.2$ for all evaluations. Maximum generation lengths follow the defaults recommended by each benchmark: 4096 tokens for agentic tasks and 16384 tokens for code tasks. We additionally report Best-of-$N$ (BoN@16) results ($N=16$ samples) in App.~\ref{app:bon}. For the multi-turn user-simulation benchmarks ($\tau^2$-Bench and VitaBench), we use Claude-4.6-Opus as the user simulator with a maximum of 300 environment steps per episode. Per-benchmark details (subtask coverage, license, evaluation metric) are in Sections~\ref{app:bfcl_detail}--\ref{app:mbpp_detail}.

\subsection{Hyperparameter Summary}
\label{app:hyperparams}
Table~\ref{tab:config} consolidates all training and evaluation hyperparameters. The key task-dependent choices are the divergence function ($\jsd_{0.5}$ for agentic tasks, $D_{\kl}^{\leftarrow}$ for code) and the maximum sequence length (4096 vs.\ 16384 tokens). Decoding hyperparameters (temperature $0.7$, top-$p$ $0.8$, presence penalty $1.2$) follow the Qwen3 technical report's recommended settings for non-thinking mode; maximum generation lengths follow each benchmark's recommended defaults.

\begin{table}[t!]
\caption{Complete training and evaluation configuration.}
\label{tab:config}
\centering
\scriptsize
\begin{tabular}{lcc}
\toprule
\textbf{Parameter} & \textbf{Agentic Tasks} & \textbf{Code Tasks} \\
\midrule
\multicolumn{3}{l}{\textit{Training}} \\
Optimizer & AdamW ($\beta_1{=}0.9$, $\beta_2{=}0.999$) & AdamW ($\beta_1{=}0.9$, $\beta_2{=}0.999$) \\
Weight decay & 0.01 & 0.01 \\
Learning rate & $1{\times}10^{-5}$, cosine, 5\% warmup & $1{\times}10^{-5}$, cosine, 5\% warmup \\
Effective batch size & 128 & 128 \\
Gradient clipping & 1.0 & 1.0 \\
Max sequence length & 4096 & 16384 \\
Data & ToolACE (${\sim}$16K, step-split) & OpenThoughts3 (30K) \\
Thinking mode & Non-thinking & Non-thinking \\
\midrule
\multicolumn{3}{l}{\textit{Distillation}} \\
Teachers ($K$) & 2 & 2 \\
Debate rounds ($R$) & 2 & 2 \\
Debate temperature & 0.7 & 0.7 \\
Divergence ($D$) & $\jsd_{0.5}$ & $D_{\kl}^{\leftarrow}$ \\
Logit vocabulary & Full vocabulary & Full vocabulary \\
Logit precision & BF16 & BF16 \\
\midrule
\multicolumn{3}{l}{\textit{Infrastructure}} \\
Training GPUs & 8$\times$ NVIDIA H20 (ZeRO-3) & 8$\times$ NVIDIA H20 (ZeRO-3) \\
Teacher debate serving & vLLM (single H20) & vLLM (single H20) \\
Logit extraction & Sidecar (single H20) & Sidecar (single H20) \\
\midrule
\multicolumn{3}{l}{\textit{Evaluation}} \\
Temperature & 0.7 & 0.7 \\
Top-$p$ & 0.8 & 0.8 \\
Presence penalty & 1.2 & 1.2 \\
Max generation length & 4096 & 16384 \\
Decoding & Sampling & Sampling (pass@1) \\
BoN & -- & $N{=}16$ \\
\bottomrule
\end{tabular}
\end{table}

\subsection{Benchmark: BFCL-v4}
\label{app:bfcl_detail}
\textbf{Berkeley Function-Calling Leaderboard v4 (BFCL-v4)}~\citep{yan2024bfcl} evaluates LLMs' function-calling and tool-use capabilities through programmatically verified API invocations. The benchmark covers eight subtasks designed to stress different aspects of structured output:
\begin{itemize}[leftmargin=1.4em, itemsep=1pt, topsep=1pt]
  \item \textbf{Simple}: a single function call to a single API with explicit arguments;
  \item \textbf{Multiple}: selecting and calling one of several available functions based on the user query;
  \item \textbf{Parallel}: emitting concurrent calls to one function with different argument sets;
  \item \textbf{Parallel-Multiple}: emitting concurrent calls across multiple distinct functions;
  \item \textbf{Base}: multi-turn function calling within an ongoing dialogue (state preservation across turns);
  \item \textbf{Miss-Function}: the model must respond appropriately when no available function matches the user's intent;
  \item \textbf{Miss-Parameter}: the model must request missing required parameters rather than hallucinating values;
  \item \textbf{Long-Context}: function calling under extended context length, used as a stress test of retrieval-and-call ability.
\end{itemize}
Evaluation uses AST-based correctness checking on the emitted function call schema, comparing argument names, types, and values against ground truth. The reported BFCL-v4 score in Table~\ref{tab:main} is the unweighted mean across all eight subtasks. Available at \url{https://github.com/ShishirPatil/gorilla/tree/main/berkeley-function-call-leaderboard} under Apache 2.0 license.

\subsection{Benchmark: \texorpdfstring{$\tau^2$}{tau2}-Bench}
\label{app:tau_detail}
\textbf{$\tau^2$-Bench}~\citep{barres2025tau2bench} extends the original $\tau$-bench (which covered airline + retail) with a new dual-control \textbf{telecom} domain and a compositional task generator. It is a multi-turn agentic benchmark for customer-service interactions, organized into three industry domains:
\begin{itemize}[leftmargin=1.4em, itemsep=1pt, topsep=1pt]
  \item \textbf{Retail}: e-commerce account management (order returns, exchanges, refunds);
  \item \textbf{Airline}: flight reservation modifications, seat selection, fare changes;
  \item \textbf{Telecom}: account management, service plan changes, troubleshooting flows.
\end{itemize}
Each task gives the agent a domain-specific toolset (e.g., \texttt{update\_order}, \texttt{modify\_reservation}) and a domain policy document. The agent must execute the user's request via tool calls while strictly adhering to the policy and tracking evolving user requirements over multiple turns. We use \textbf{Claude-4.6-Opus} as the user simulator and cap each episode at 300 environment steps. Success is graded by the official $\tau^2$-bench scorer (exact match of the resulting database state against the ground-truth target state). \textbf{Scoring protocol.} We follow the benchmark's default accuracy definition: episodes that fail to execute (e.g., simulator API errors, malformed tool calls causing crashes) are excluded from the denominator, and we only score normally-executed episodes; each evaluation run is required to have $\geq 90\%$ normally-executed episodes to be counted as valid (otherwise we re-run that configuration). The reported $\tau^2$ score in Table~\ref{tab:main} is the pass@1 success rate averaged over the three domains. Available at \url{https://github.com/sierra-research/tau2-bench} under MIT license. We use the latest release of the repository, which includes 75+ task fixes for airline / retail; the telecom domain is the original $\tau^2$ implementation. The voice and banking-knowledge modalities of the latest release are not used.

\subsection{Benchmark: VitaBench}
\label{app:vita_detail}
\textbf{VitaBench}~\citep{he2025vitabench} is a versatile interactive task benchmark targeting real-world life-service scenarios. It exposes 66 tools across three domains:
\begin{itemize}[leftmargin=1.4em, itemsep=1pt, topsep=1pt]
  \item \textbf{Food delivery}: restaurant search, menu navigation, order placement, delivery tracking;
  \item \textbf{In-store consumption}: in-shop ordering, payment, inventory checks, loyalty programs;
  \item \textbf{Online travel agency (OTA)}: flight/hotel/itinerary search and booking, modification, refunds.
\end{itemize}
The benchmark is constructed via a domain-policy-free framework that flexibly composes scenarios and tools, comprising 100 cross-scenario tasks (used as the main evaluation) plus 300 single-scenario tasks. Each task is derived from real user requests and requires the agent to reason across temporal and spatial dimensions, use complex tool sets, proactively clarify ambiguous instructions, and track shifting user intent throughout multi-turn conversations. Evaluation uses the official VitaBench rubric-based sliding-window scorer that supports robust assessment under stochastic interactions and diverse solution paths. We use \textbf{Claude-4.6-Opus} as the user simulator and cap each episode at 300 environment steps. \textbf{Scoring protocol.} We follow the benchmark's default accuracy definition: episodes that fail to execute (e.g., simulator API errors, malformed tool calls causing crashes) are excluded from the denominator, and we only score normally-executed episodes; each evaluation run is required to have $\geq 90\%$ normally-executed episodes to be counted as valid (otherwise we re-run that configuration). The reported Vita score in Table~\ref{tab:main} is the success rate averaged over the three scenarios. Available at \url{https://vitabench.github.io/}.

\subsection{Benchmark: LiveCodeBench v6}
\label{app:lcb_detail}
\textbf{LiveCodeBench v6 (LCB-v6)}~\citep{jain2024livecodebench} is a contamination-free code generation benchmark. Problems are continuously sourced from competitive programming platforms (LeetCode, AtCoder, Codeforces) with timestamp-based filtering: each release window includes only problems published \emph{after} the cutoff date of all evaluated models, eliminating training-data contamination by construction. Version v6 is the latest release at the time of our experiments. Each problem provides a natural-language task description and a hidden test suite (typically 20--100 hidden tests). Following the official protocol, generations are run through the hidden tests under a per-test execution timeout. We report \textbf{pass@1} (single-sample success rate, decoding at temperature 0.7 / top-$p$ 0.8) as the primary metric and \textbf{BoN@16} (best of 16 sampled solutions per problem at temperature 0.7) as a multi-sample measure. All scores are means over 16 independent random seeds with std reported as subscript. Available at \url{https://github.com/LiveCodeBench/LiveCodeBench} under Apache 2.0 license.

\subsection{Benchmark: MBPP+}
\label{app:mbpp_detail}
\textbf{MBPP+}~\citep{liu2023evalplus} is the EvalPlus extension of \textbf{MBPP}~\citep{austin2021mbpp} (Mostly Basic Python Problems). EvalPlus augments the original MBPP with extensive test cases generated through differential testing and mutation-based input synthesis (approximately $5\times$ more tests per problem on average). The augmented test suite catches code-correctness failures (e.g., off-by-one errors, edge-case mishandling) that the original MBPP tests miss, providing a substantially more reliable signal of functional correctness. The benchmark contains 378 sanitized programming problems covering basic Python data structures, algorithms, and string/list manipulation; each problem provides a natural-language description, a function signature, and the augmented test suite. We report \textbf{pass@1} and \textbf{BoN@16} on the same protocol as LCB-v6 (16 seeds, std subscript). Available at \url{https://github.com/evalplus/evalplus} under Apache 2.0 license.

\section{Additional Experiment Results}
\label{app:additional_results}

This appendix consolidates supplementary empirical evidence; each subsection indicates which research question (Sec.~\ref{sec:exp}) it supports.

\subsection{Component Analysis: Full Numbers (RQ3)}
\label{app:component_full}

Table~\ref{tab:components} reports the per-component ablation underlying Figure~\ref{fig:ablation_compact}(a) on the 14B+8B$\to$4B configuration. Components are added incrementally, isolating the contribution of each.

\begin{table}[t!]
\caption{\textbf{Component analysis} on \ours{}'s most-improved setting (14B+8B$\to$4B). Components are added incrementally: \textbf{Multi-T} = multi-teacher aggregation; \textbf{On-Policy} = student-rollout supervision (vs.\ off-policy SeqKD); \textbf{Debate} = inter-teacher debate transcript as privileged information; \textbf{Conf.\ Wt.} = post-debate confidence-based teacher weighting (Eq.~\ref{eq:consensus_score}). \textbf{Bold} = best, \underline{underline} = second-best.}
\label{tab:components}
\centering
\small
\setlength{\tabcolsep}{6pt}
\begin{tabular}{l cccc ccc}
\toprule
\textbf{Method} & Multi-T & On-Policy & Debate & Conf.\ Wt. & Ag-Avg & Co-Avg & Avg \\
\midrule
Base &  &  &  &  & 18.65 & 41.50 & 27.79 \\
MT-SeqKD & \checkmark &  &  &  & 21.48 & 40.77 & 29.19 \\
OPD &  & \checkmark &  &  & 23.26 & 44.41 & 31.72 \\
MT-OPD & \checkmark & \checkmark &  &  & \underline{24.14} & 41.97 & 31.27 \\
\ours{} w/o Conf.\ Wt. & \checkmark & \checkmark & \checkmark &  & 23.95 & \underline{46.60} & \underline{33.01} \\
\textbf{\ours{}} & \checkmark & \checkmark & \checkmark & \checkmark & \textbf{25.69} & \textbf{48.12} & \textbf{34.66} \\
\bottomrule
\end{tabular}
\end{table}

\paragraph{Robustness of confidence weighting.}
The Conf.\ Wt.\ increment ($+1.74\%$ Ag-Avg, $+1.52\%$ Co-Avg over the uniform-weight ablation) does not require teachers' self-reported confidence to be well calibrated. Two structural reasons. First, the softmax in Eq.~\ref{eq:consensus_score} cancels any teacher-shared bias in absolute confidence scale: only the \emph{relative} ordering of $\{c_k\}$ enters the gradient. Second, the increment requires only that relative confidence ordering correlates with relative quality on average, not that any single $c_k$ is accurate; the empirical gain over the uniform-weight ablation indicates this correlation holds in practice across the $\sim$16K agentic and $30$K code training instances. In the worst case where confidence becomes uninformative (uniform $\{c_k\}$), Eq.~\ref{eq:consensus_score} degenerates to $w_k = 1/K$ and the loss reduces to MT-OPD, which itself already exceeds Base on every configuration of Table~\ref{tab:main}. A systematic per-step calibration study is left to future work.

\subsection{Divergence Ablation: Full Numbers (RQ4)}
\label{app:divergence_full}

Table~\ref{tab:divergence} reports the full divergence ablation summarized in Sec.~\ref{sec:ablation}. We fix all hyperparameters except $D$ on the 14B+8B$\to$4B configuration; the JSD-agentic and Rev-KL-code rows correspond to the default \ours{} configuration and therefore match the main \ours{} row in Table~\ref{tab:main}, while the off-diagonal rows isolate each non-default $D$. The dominance pattern inverts along the agentic/code split, as predicted by Remark~\ref{rem:task_dep}.

\begin{table}[t!]
\caption{\textbf{Divergence ablation} on \ours{} (14B+8B$\to$4B). JSD leads on agentic tasks; reverse KL leads on code, validating the task-dependent selection (Remark~\ref{rem:task_dep}). Code metrics report pass@1 averaged over 16 seeds (subscript: std), matching the main-table protocol; the JSD-agentic and Rev-KL-code rows correspond to the default \ours{} configuration in Table~\ref{tab:main}.}
\label{tab:divergence}
\centering
\small
\begin{tabular}{l ccc c cc c}
\toprule
& \multicolumn{4}{c}{\textbf{Agentic} ($D{=}\jsd$)} & \multicolumn{3}{c}{\textbf{Code} ($D{=}D_{\kl}^{\leftarrow}$)} \\
\cmidrule(lr){2-5} \cmidrule(lr){6-8}
\textbf{Divergence $D$} & BFCL-v4 & $\tau^2$ & Vita & Ag-Avg & LCB-v6 & MBPP+ & Co-Avg \\
\midrule
Fwd KL ($D_{\kl}^{\to}$) & 23.70 & \underline{26.24} & 15.21 & 21.72 & 24.29{\tiny$_{\pm 0.73}$} & 57.03{\tiny$_{\pm 1.13}$} & 40.66 \\
Rev KL ($D_{\kl}^{\leftarrow}$) & \underline{25.02} & 25.85 & \underline{15.62} & \underline{22.16} & \textbf{29.83}{\tiny$_{\pm 0.56}$} & \textbf{66.42}{\tiny$_{\pm 0.92}$} & \textbf{48.12} \\
JSD ($\jsd_\beta$) & \textbf{27.08} & \textbf{31.46} & \textbf{18.53} & \textbf{25.69} & \underline{24.92}{\tiny$_{\pm 0.92}$} & \underline{64.50}{\tiny$_{\pm 0.95}$} & \underline{44.71} \\
\bottomrule
\end{tabular}
\end{table}

\subsection{Student Outperforms Teachers on LCB-v6 (RQ1)}
\label{app:lcb_surpass}

Figure~\ref{fig:lcb_bubble} reports an accuracy--efficiency view of the 14B+8B$\to$4B configuration on LCB-v6 (pass@1 over 16 seeds, mean $\pm$ std). We use LCB-v6 rather than MBPP+ for this efficiency view because MBPP+ tasks are short and elicit very few output tokens per problem, leaving the token-cost axis essentially flat across methods and uninformative; LCB-v6 problems generate substantially longer solutions, exposing meaningful per-method token-budget differences.

The 4B student trained with \ours{} reaches $29.83\%$ pass@1, exceeding the 14B teacher ($25.57\%$) and the 8B teacher ($20.93\%$): an absolute gain of $+4.26\%$ over the strongest teacher despite using $3.5\times$ fewer parameters and at competitive token cost vs.\ all distillation baselines. The advantage extends beyond pass@1: under BoN@16 the same 4B student reaches $43.43\%$, outperforming the 14B teacher's BoN@16 ($33.14\%$) by $+10.29\%$. No other distillation baseline crosses the 14B-teacher pass@1 ceiling: OPD lands at $25.00\%$, MT-OPD at $23.29\%$, MT-SeqKD at $23.18\%$. The result shows that confidence-weighted multi-teacher debate produces supervision that exceeds the strongest single teacher in its own pool, both in single-shot and multi-sample selection regimes.

\begin{figure}[t!]
\centering
\includegraphics[width=\linewidth]{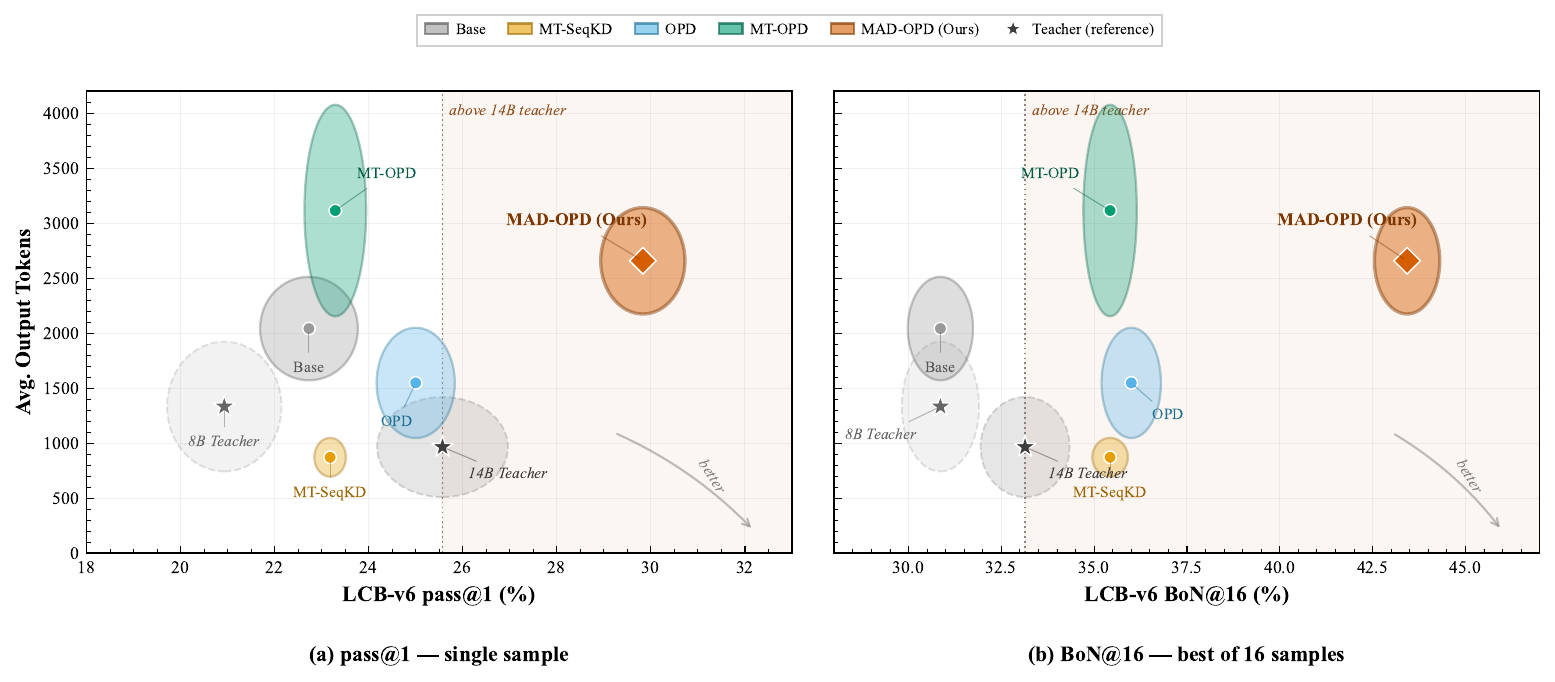}
\caption{\textbf{Accuracy--efficiency on LCB-v6 (14B+8B$\to$4B, 16 seeds).} Both panels share the y-axis (avg.\ output tokens, lower is better) for direct token-cost comparison. Each ellipse spans $\pm 1$ standard deviation; markers are seed means; teacher reference points (gray stars) are dashed. The vertical dotted line and shaded region mark the \emph{14B-teacher ceiling}. \textbf{(a)~pass@1 (single sample)}: only \ours{} crosses the $25.57\%$ ceiling, reaching $29.83\%$ ($+4.26\%$ over the strongest teacher, $+4.83\%$ over OPD). \textbf{(b)~BoN@16 (best of 16 samples)}: \ours{} again clears the $33.14\%$ teacher ceiling, reaching $43.43\%$ ($+10.29\%$ over the 14B teacher, $+7.43\%$ over OPD); the dominance is more pronounced under multi-sample selection. In both panels, token cost is on par with MT-OPD and below MT-SeqKD's regime, ruling out token inflation as the source of the gain.}
\label{fig:lcb_bubble}
\end{figure}

\subsection{Best-of-N Results for Code Benchmarks (RQ1)}
\label{app:bon}

In addition to pass@1, we evaluate code generation quality using the Best-of-$N$ (BoN) protocol with $N{=}16$. For each problem, we sample 16 candidate solutions (temperature $0.7$) and select the one that passes the most test cases. Table~\ref{tab:bon_results} reports BoN@16 on both LCB-v6 and MBPP+ across all six configurations. \emph{On LCB-v6}, \ours{} achieves the best (or tied-best) BoN@16 in 5 of 6 configurations; MT-SeqKD leads only on 32B+30B$\to$8B (by a $1.14$-point margin) and ties with \ours{} on Qwen3.5 27B+9B$\to$4B, likely because its off-policy training produces more diverse outputs that benefit from best-of-$N$ selection. \emph{On MBPP+}, \ours{} ranks first in every configuration.

\begin{table}[t!]
\caption{\textbf{Code BoN@16 results} on LCB-v6 and MBPP+ across all configurations ($N{=}16$ samples per problem). \textbf{Bold}: best per configuration per benchmark; \underline{underline}: second best. See Table~\ref{tab:teacher_base} for the corresponding teacher BoN@16 reference numbers.}
\label{tab:bon_results}
\centering
\renewcommand{\arraystretch}{0.92}
\scriptsize
\setlength{\tabcolsep}{4pt}
\begin{tabular*}{\textwidth}{@{\extracolsep{\fill}} l cc ccc c @{}}
\toprule
& \multicolumn{2}{c}{\textbf{Qwen3 14B{+}8B}} & \multicolumn{3}{c}{\textbf{Qwen3 32B{+}30B\text{-}A3B}} & \textbf{Qwen3.5 27B{+}9B} \\
\cmidrule(lr){2-3} \cmidrule(lr){4-6} \cmidrule(lr){7-7}
\textbf{Method} & $\to$1.7B & $\to$4B & $\to$4B & $\to$8B & $\to$14B & $\to$4B \\
\midrule
\multicolumn{7}{l}{\cellcolor{gray!8}\textit{LCB-v6 BoN@16}} \\
\cmidrule{1-7}
MT-SeqKD   & \underline{28.00} & 35.43 & 36.57 & \textbf{40.57} & 42.86 & \textbf{49.71} \\
OPD        & \textbf{31.43} & \underline{36.00} & 36.00 & 36.00 & 38.92 & 41.14 \\
MT-OPD     & 26.29 & 35.43 & \underline{38.86} & \underline{39.43} & \underline{47.43} & \underline{49.41} \\
\textbf{\ours{}} & \textbf{31.43} & \textbf{43.43} & \textbf{43.43} & \underline{39.43} & \textbf{49.71} & \textbf{49.71} \\
\midrule
\multicolumn{7}{l}{\cellcolor{gray!8}\textit{MBPP+ BoN@16}} \\
\cmidrule{1-7}
MT-SeqKD   & 61.90 & 70.37 & \underline{71.62} & 69.58 & 77.78 & 66.67 \\
OPD        & \underline{66.67} & \underline{73.81} & 67.73 & 73.54 & \underline{80.95} & 75.13 \\
MT-OPD     & 65.18 & 67.72 & 69.58 & \underline{75.40} & 80.42 & \underline{81.48} \\
\textbf{\ours{}} & \textbf{67.73} & \textbf{74.34} & \textbf{73.81} & \textbf{76.46} & \textbf{81.69} & \textbf{81.69} \\
\bottomrule
\end{tabular*}
\end{table}

\subsection{Training Loss Stability under Different Divergences (RQ4)}
\label{app:loss_stability}

Figure~\ref{fig:loss_stability} visualizes the per-step training loss for \ours{} and OPD on agentic and code tasks under each of the three divergences (Forward KL, JSD, Reverse KL). \textbf{Agentic.} Under reverse KL, $\mathcal{L}_{\mathrm{MAD}}$ swings across two orders of magnitude: the empirical signature of the unbounded per-token logit gradient derived in Prop.~\ref{prop:trajectory_var}.2, which fires whenever the privileged teacher disprefers a student-sampled token. JSD stays bounded throughout, consistent with the per-position gradient bound of Prop.~\ref{prop:trajectory_var}.1. \textbf{Code.} The dynamics invert: reverse KL converges smoothly because mode concentration (Lemma~\ref{lem:rev_kl_mode}, Prop.~\ref{prop:rev_kl_code}.1) matches the single-mode structure that pass@1 rewards, whereas JSD's interpolating geometry retains mass on secondary modes (Prop.~\ref{prop:rev_kl_code}.3) and produces incoherent splices across implementations.

\begin{figure}[t!]
\centering
\includegraphics[width=\linewidth]{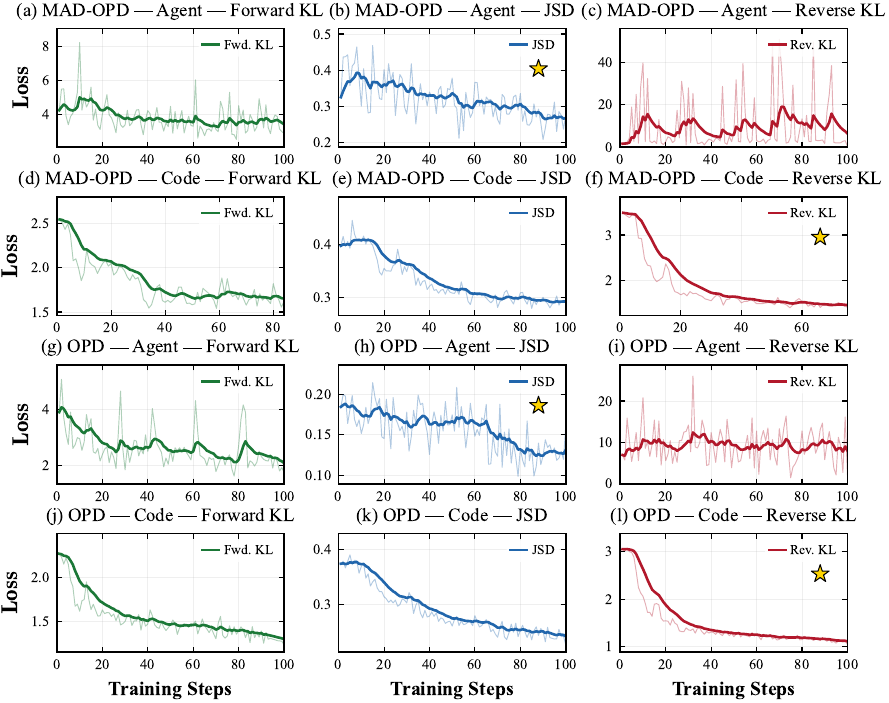}
\caption{\textbf{Training loss curves} for \ours{} (rows~1--2) and OPD (rows~3--4) on agentic and code tasks under three divergences (Forward KL, JSD, Reverse KL), 14B+8B$\to$4B. X-axes are gradient steps (gradient accumulation $=8$, so absolute loss values are roughly $8\times$ the per-micro-batch value), truncated to the first $100$. Each panel overlays the per-step training loss (light) with an exponential moving average (solid) using smoothing factor $\alpha = 0.15$ (i.e.\ $s_t = 0.15\,y_t + 0.85\,s_{t-1}$, $s_0 = y_0$). \textbf{Agentic panels (rows~1, 3)}: reverse KL (column~3) exhibits high-variance spikes consistent with the unbounded per-token logit gradient under privileged $p$--$q$ gaps (Prop.~\ref{prop:trajectory_var}.2), whereas JSD (column~2) remains stable (Prop.~\ref{prop:trajectory_var}.1). \textbf{Code panels (rows~2, 4)}: reverse KL converges smoothly, matching the mode-seeking behavior favored by single-mode code distributions.}
\label{fig:loss_stability}
\end{figure}

\subsection{Effect of Debate Rounds \texorpdfstring{$R$}{R} (RQ3)}
\label{app:rounds_ablation}

Figure~\ref{fig:rounds_ablation} ablates the number of debate rounds $R \in \{0, 1, 2, 3\}$ on the 14B+8B$\to$4B configuration across all five benchmarks. $R{=}0$ removes teacher generation entirely (single-teacher OPD); $R{=}1$ uses only the independent round and recovers MT-OPD under uniform weights; $R{\geq}2$ activates inter-teacher revision. $R{=}2$ emerges as the empirical optimum on every benchmark, while $R{=}3$ overshoots: the extra round inflates the prompt context and introduces noise that erodes the gains (Ag-Avg drops by $-8.27\%$ relative to $R{=}2$).

\begin{figure}[t!]
\centering
\includegraphics[width=\linewidth]{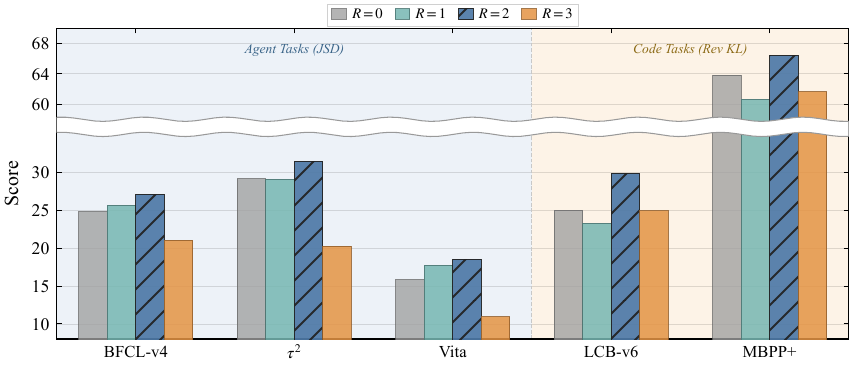}
\caption{\textbf{Effect of debate rounds $R$} on the 14B+8B$\to$4B configuration. $R{=}0$: no teacher generation (single-teacher OPD). $R{=}1$: independent only (MT-OPD). $R{=}2$ (hatched, default): one inter-teacher revision round. $R{=}3$ degrades sharply on agentic tasks as the extra round inflates the prompt context. Background shading separates agentic tasks (left, JSD) from code tasks (right, reverse KL). The wavy break compresses the y-axis between $35$ and $58$ to show both regimes at usable resolution.}
\label{fig:rounds_ablation}
\end{figure}

\subsection{Teacher Base Performance Reference}
\label{app:teacher_base}

Table~\ref{tab:teacher_base} consolidates the undistilled (Base) performance of every teacher used in the six configurations of Sec.~\ref{sec:setup}. These numbers serve as the single-teacher upper bound that our \textbf{OPD} baseline already attains and that \ours{} aims to exceed via debate. The 30B-A3B teacher refers to the \texttt{Qwen3-30B-A3B-Instruct-2507} release; all other teachers use the standard Instruct checkpoints of the corresponding size (App.~\ref{app:details}).

\begin{table}[t!]
\caption{\textbf{Teacher base performance.} Single-model evaluation of the six teachers across all five benchmarks. $\tau^2$ and Vita report the three-domain means defined in App.~\ref{app:tau_detail} / App.~\ref{app:vita_detail}. Code metrics: pass@1 with std as subscript over 16 seeds, and BoN@16 with $N{=}16$ samples per problem. The Qwen3.5 family exhibits a strong baseline on these tasks, leaving smaller per-method headroom for the Qwen3.5 27B+9B$\to$4B configuration discussed in Sec.~\ref{sec:scaling}.}
\label{tab:teacher_base}
\centering
\small
\setlength{\tabcolsep}{4pt}
\begin{tabular}{l ccc cc cc}
\toprule
& \multicolumn{3}{c}{\textbf{Agentic}} & \multicolumn{2}{c}{\textbf{Code pass@1}} & \multicolumn{2}{c}{\textbf{Code BoN@16}} \\
\cmidrule(lr){2-4} \cmidrule(lr){5-6} \cmidrule(lr){7-8}
\textbf{Teacher} & BFCL-v4 & $\tau^2$ & Vita & LCB-v6 & MBPP+ & LCB-v6 & MBPP+ \\
\midrule
\multicolumn{8}{l}{\cellcolor{gray!8}\textit{Qwen3 family}} \\
\cmidrule{1-8}
Qwen3-8B            & 28.94 & 30.08 & 17.16 & 20.93{\tiny$_{\pm 0.76}$} & 63.56{\tiny$_{\pm 0.61}$} & 30.86 & 70.90 \\
Qwen3-14B           & 32.46 & 33.70 & 21.70 & 25.57{\tiny$_{\pm 0.87}$} & 75.13{\tiny$_{\pm 0.78}$} & 33.14 & 80.16 \\
Qwen3-30B-A3B$^{\dagger}$  & 35.72 & 48.88 & 24.90 & 36.00{\tiny$_{\pm 0.92}$} & 77.87{\tiny$_{\pm 0.81}$} & 51.43 & 80.95 \\
Qwen3-32B           & 38.01 & 48.39 & 25.67 & 27.43{\tiny$_{\pm 0.84}$} & 78.21{\tiny$_{\pm 0.65}$} & 36.00 & 81.75 \\
\midrule
\multicolumn{8}{l}{\cellcolor{gray!8}\textit{Qwen3.5 family}} \\
\cmidrule{1-8}
Qwen3.5-9B          & 36.97 & 49.33 & 11.72 & 39.43{\tiny$_{\pm 1.13}$} & 63.36{\tiny$_{\pm 1.06}$} & 55.43 & 79.37 \\
Qwen3.5-27B         & 49.04 & 56.54 & 11.87 & 43.43{\tiny$_{\pm 0.48}$} & 72.21{\tiny$_{\pm 0.97}$} & 65.71 & 81.22 \\
\bottomrule
\end{tabular}\\[2pt]
{\footnotesize $^{\dagger}$\,\texttt{Qwen3-30B-A3B-Instruct-2507} release.}
\end{table}

\section{Detailed Proofs}
\label{app:derivations}

This appendix provides complete proofs for the lemmas and propositions in Section~\ref{sec:prelim}. \S\ref{app:opd_grad} derives the dense OPD gradient (no REINFORCE term). \S\ref{app:proof_lem_jsd}--\ref{app:proof_lem_revkl} prove the per-divergence Lemmas~\ref{lem:jsd_bound}~(JSD loss/gradient bounds) and~\ref{lem:rev_kl_mode}~(reverse-KL mode concentration). \S\ref{app:proof_prop_traj} proves Proposition~\ref{prop:trajectory_var}, comparing the per-token logit-gradient behaviour of all three divergences under privileged $p$--$q$ gaps. \S\ref{app:proof_prop_code} proves Proposition~\ref{prop:rev_kl_code}, the mode-coverage analysis for code generation. \S\ref{app:fwd_kl_instability} supplements with forward-KL's mode-covering failure mode in agentic settings.

\subsection{OPD Gradient Derivation}
\label{app:opd_grad}

\begin{proof}
At each position $t$, $D(p \| q)$ depends on $\theta$ only through $q = \pi_\theta(\cdot | x, \hat{y}_{<t})$; the privileged teacher $p = \pi^*(\cdot | x, c, \hat{y}_{<t})$ is a stop-gradient target. Teacher-forcing fixes $\hat{y}_{<t}$ as non-differentiable context, dropping the REINFORCE score-function term $G \cdot \nabla \log \pi$ (with $G$ the trajectory return), so that
\begin{equation}
\nabla_\theta \mathcal{L}_{\textsc{opd}}(\theta) \;=\; \E_{x,\,\hat{y}\sim\pi_\theta}\!\left[\frac{1}{|\hat{y}|}\sum_{t=1}^{|\hat{y}|} \nabla_\theta D(p \,\|\, q)\right].
\label{eq:opd_dense_grad}
\end{equation}
As an illustration, decomposing forward KL as $D_{\kl}^{\to}(p\|q) = H(p,q) - H(p)$ (with $H(p)$ constant in $z$) and applying the softmax Jacobian $\partial q(v)/\partial z_i = q(v)(\delta_{iv}{-}q(i))$ collapses the chain rule to
\begin{equation}
\nabla_{z_i} D_{\kl}^{\to}(p\|q) \;=\; -\sum_v p(v)(\delta_{iv}{-}q(i)) \;\overset{\sum p = 1}{=}\; q(i) - p(i),
\label{eq:fwd_kl_grad_form}
\end{equation}
the standard ``$q-p$'' cross-entropy gradient with $|q(i)-p(i)| \leq 1$ uniformly. Stability thus reduces to the per-position bound $\|\nabla_{z_t} D\|_\infty$, derived per divergence in \S\ref{app:proof_lem_jsd}--\ref{app:proof_prop_traj}.
\end{proof}

\subsection{Proof of Lemma~\ref{lem:jsd_bound} (JSD Bounded Loss and Bounded Logit Gradient)}
\label{app:proof_lem_jsd}

\textbf{Restatement.} \emph{For any distributions $p, q$ over $\mathcal{V}$ and any $\beta \in (0,1)$, (i) $\jsd_\beta(p\|q) \in [0, H(\beta)]$ with $H(\beta) = -\beta\log\beta - (1{-}\beta)\log(1{-}\beta)$; (ii) $\|\nabla_z \jsd_\beta(p\|q)\|_\infty \leq 2$ uniformly over all $p, q$.}

\begin{proof}
Let $m = \beta p + (1{-}\beta)q$ and $q(v) = \mathrm{softmax}(z)_v$, with $0\log 0 := 0$; softmax gives $q(v) > 0$, hence $m(v) \geq (1{-}\beta)q(v) > 0$ everywhere.

\textbf{Part 1: Loss boundedness.} From $m(v) \geq \beta p(v)$ we have $\tfrac{p(v)}{m(v)} \leq \tfrac{1}{\beta}$, hence $\log\tfrac{p(v)}{m(v)} \leq \log\tfrac{1}{\beta}$, so
\begin{equation}
D_{\kl}(p\|m) \;=\; \sum_v p(v)\log\tfrac{p(v)}{m(v)} \;\leq\; \sum_v p(v)\log\tfrac{1}{\beta} \;=\; \log\tfrac{1}{\beta},
\label{eq:per_kl_bounds}
\end{equation}
and symmetrically $D_{\kl}(q\|m) \leq \log\tfrac{1}{1{-}\beta}$ from $m(v) \geq (1{-}\beta)q(v)$. Combining and writing $H(\beta) := -\beta\log\beta - (1{-}\beta)\log(1{-}\beta)$,
\begin{equation}
\jsd_\beta(p\|q) \;=\; \beta\,D_{\kl}(p\|m) + (1{-}\beta)\,D_{\kl}(q\|m) \;\leq\; H(\beta),
\label{eq:jsd_loss_bound}
\end{equation}
so $\jsd_{0.5} \leq \log 2$; the lower bound $\jsd_\beta \geq 0$ follows from $D_{\kl} \geq 0$.

\textbf{Part 2: Gradient boundedness.}

\emph{Step 1: Decomposition.} Logits $z$ enter $\jsd_\beta$ through $q$ only ($p$ fixed); by linearity,
\begin{equation}
\nabla_{z_i}\jsd_\beta \;=\; \beta\,\nabla_{z_i} D_{\kl}(p\|m) \;+\; (1{-}\beta)\,\nabla_{z_i} D_{\kl}(q\|m), \label{eq:jsd_decomp}
\end{equation}
with softmax Jacobian
\begin{equation}
\frac{\partial q(v)}{\partial z_i} = q(v)(\delta_{iv}{-}q(i)), \qquad \frac{\partial m(v)}{\partial z_i} = (1{-}\beta)\,q(v)(\delta_{iv}{-}q(i)). \label{eq:softmax_jac}
\end{equation}

\emph{Step 2: $\nabla_{z_i} D_{\kl}(p\|m)$.} Only $-\sum_v p(v)\log m(v)$ depends on $z$. Chain rule with $\nabla_{z_i}\log m(v) = \tfrac{1}{m(v)}\tfrac{\partial m(v)}{\partial z_i}$ and \eqref{eq:softmax_jac}:
\begin{equation}
\nabla_{z_i} D_{\kl}(p\|m) \;=\; -\sum_v \tfrac{p(v)}{m(v)} \cdot \tfrac{\partial m(v)}{\partial z_i} \;=\; -(1{-}\beta)\sum_v \tfrac{p(v)\,q(v)\,(\delta_{iv}{-}q(i))}{m(v)}. \label{eq:jsd_term_p}
\end{equation}

\emph{Step 3: $\nabla_{z_i} D_{\kl}(q\|m)$.} Expanding $D_{\kl}(q\|m) = \sum_v q(v)\log\tfrac{q(v)}{m(v)}$ and applying the product rule (using $\sum_v \partial q(v)/\partial z_i = 0$):
\begin{equation}
\nabla_{z_i} D_{\kl}(q\|m) \;=\; \underbrace{\sum_v q(v)(\delta_{iv}{-}q(i))\log\tfrac{q(v)}{m(v)}}_{=:\,X} \;-\; \underbrace{(1{-}\beta)\sum_v \tfrac{q(v)^2(\delta_{iv}{-}q(i))}{m(v)}}_{=:\,Y}. \label{eq:jsd_term_q}
\end{equation}

\emph{Step 4: Bound on $|\beta\,\nabla_{z_i} D_{\kl}(p\|m)|$.} From \eqref{eq:jsd_term_p} with $|\delta_{iv}{-}q(i)| \leq 1$, then weighted AM--GM ($m \geq p^\beta q^{1-\beta}$) followed by Young's inequality:
\begin{equation}
\tfrac{p(v)q(v)}{m(v)} \;\leq\; p(v)^{1-\beta}q(v)^\beta \;\leq\; (1{-}\beta)p(v) + \beta\,q(v) \;\;\Longrightarrow\;\; \textstyle\sum_v \tfrac{p(v)q(v)}{m(v)} \leq 1.
\label{eq:amgm_young}
\end{equation}
Hence
\begin{equation}
|\beta\,\nabla_{z_i} D_{\kl}(p\|m)| \;\leq\; \beta(1{-}\beta) \;\leq\; \tfrac{1}{4}. \label{eq:term_p_bound}
\end{equation}

\emph{Step 5: Bound on $|(1{-}\beta)\,\nabla_{z_i} D_{\kl}(q\|m)|$.} Splitting the indicator in \eqref{eq:jsd_term_q} (the $v=i$ term plus $\sum_{v\neq i}$ collapsed via $\sum_v q(v) = 1$) yields
\begin{equation}
X \;=\; q(i)\log\tfrac{q(i)}{m(i)} \;-\; q(i)\,D_{\kl}(q\|m).
\label{eq:X_decomp}
\end{equation}
For the first term, case-split on $\mathrm{sign}\!\log\tfrac{q(i)}{m(i)}$:
\begin{equation}
\big|q(i)\log\tfrac{q(i)}{m(i)}\big| \;\leq\; \begin{cases}
q(i)\log\tfrac{1}{1{-}\beta} \;\leq\; \log\tfrac{1}{1{-}\beta} & \text{if } q(i) \geq m(i),\\[2pt]
m(i){-}q(i) = \beta(p(i){-}q(i)) \leq \beta \leq \log\tfrac{1}{1{-}\beta} & \text{if } q(i) \leq m(i),
\end{cases}
\label{eq:X_pieces}
\end{equation}
where the first case uses $m(i)\geq(1{-}\beta)q(i)$ and the second uses $\log x \leq x{-}1$ together with Taylor expansion. For the second term, $q(i)\,D_{\kl}(q\|m) \leq D_{\kl}(q\|m) \leq \log\tfrac{1}{1{-}\beta}$ by Part~1. Triangle inequality on \eqref{eq:X_decomp}:
\begin{equation}
|X| \;\leq\; 2\log\tfrac{1}{1{-}\beta}.
\label{eq:X_bound}
\end{equation}
For $|Y|$, $q(v)^2/m(v) \leq q(v)/(1{-}\beta)$ and $|\delta_{iv}{-}q(i)| \leq 1$ yield $|Y| \leq 1$. Triangle inequality on $X-Y$, then prefactor $(1{-}\beta)$:
\begin{equation}
|(1{-}\beta)\,\nabla_{z_i} D_{\kl}(q\|m)| \;\leq\; (1{-}\beta)\Big(2\log\tfrac{1}{1{-}\beta} + 1\Big) \;\leq\; \tfrac{2}{\sqrt{e}}, \label{eq:term_q_bound}
\end{equation}
where $\max_{x\in(0,1)} x(1 - 2\log x) = 2/\sqrt{e}$ at $x = e^{-1/2}$.

\emph{Step 6: Combining.} Triangle inequality on \eqref{eq:jsd_decomp} via \eqref{eq:term_p_bound}, \eqref{eq:term_q_bound}:
\begin{equation}
\|\nabla_z \jsd_\beta(p\|q)\|_\infty \;\leq\; \tfrac{1}{4} + \tfrac{2}{\sqrt{e}} \;\approx\; 1.46 \;<\; 2,
\label{eq:jsd_final_bound}
\end{equation}
uniformly in $\beta\in(0,1), p, q$.
\end{proof}

\subsection{Proof of Lemma~\ref{lem:rev_kl_mode} (Reverse KL Mode Concentration)}
\label{app:proof_lem_revkl}

\textbf{Restatement.} \emph{For a mixture $p = \sum_{j=1}^J \alpha_j p_j$ with $J$ well-separated modes (disjoint supports), the reverse-KL minimizer $q^* = \arg\min_q D_{\kl}(q\|p)$ over softmax-parameterized distributions concentrates on the dominant mode $p_{j^*}$, $j^* = \arg\max_j \alpha_j$, with optimal cost $D_{\kl}(q^*\|p) = -\log \alpha_{j^*}$.}

\begin{proof}
Let $\{p_j\}$ have pairwise disjoint supports $\mathcal{S}_j \subset \mathcal{V}$ and write the reverse KL in cross-entropy form:
\begin{equation}
D_{\kl}(q\|p) \;=\; H(q,p) - H(q), \qquad H(q,p) := -\sum_v q(v)\log p(v),
\label{eq:revkl_expand}
\end{equation}
where $H(q,p) \geq 0$ is the cross-entropy and $H(q) \geq 0$ the entropy. For $v \notin \bigcup_j \mathcal{S}_j$ we have $p(v) = 0$, so any mass $q(v) > 0$ outside the mode supports makes $H(q,p) = +\infty$; hence $\mathrm{supp}(q^*) \subseteq \bigcup_j \mathcal{S}_j$. Let $\mu_j := \sum_{v \in \mathcal{S}_j} q(v)$ ($\sum_j \mu_j = 1$); within each mode the optimal conditional matches $p_j$, reducing the objective to
\begin{equation}
D_{\kl}(q^*\|p) \;=\; \sum_j \mu_j \log\frac{\mu_j}{\alpha_j} \;=\; D_{\kl}(\mu \| \alpha).
\label{eq:revkl_reduced}
\end{equation}
While an unconstrained categorical distribution could match $\mu = \alpha$ at a single step, the reverse-KL loss is itself weighted by the student's own probability $q$: any token with $q(v) \to 0$ contributes $q(v)\log\tfrac{q(v)}{p(v)} \to 0$ (under $0\log 0 := 0$), so the gradient signal to recover secondary modes vanishes. Once the student slightly favors one mode, optimization self-reinforces toward it; \citet{gu2024minillm} establish this mode-seeking dynamic for LM distillation. The concentrated solution ($\mu_{j^*} \to 1$ at $j^* = \arg\max_j \alpha_j$) yields
\begin{equation}
D_{\kl}(q^*\|p) \;=\; -\log\alpha_{j^*}.
\label{eq:revkl_optimum}
\end{equation}
\end{proof}

\subsection{Proof of Proposition~\ref{prop:trajectory_var} (Per-Token Gradient Stability)}
\label{app:proof_prop_traj}

\textbf{Restatement.} \emph{Under the dense OPD gradient estimator (App.~\ref{app:opd_grad}), the per-token logit gradient satisfies: (i) JSD: $\|\nabla_{z_t}\jsd_\beta(p\|q)\|_\infty \leq 2$ worst-case, hence $\|\nabla_z \mathcal{L}_{\jsd}(\tau)\|_\infty \leq 2$ irrespective of trajectory length; (ii) Reverse KL: $\|\nabla_{z_t} D_{\kl}^{\leftarrow}(q\|p)\|_\infty$ is unbounded as $p(i)\to 0$ with $q(i)>0$; (iii) Forward KL: $\|\nabla_{z_t} D_{\kl}^{\to}(p\|q)\|_\infty \leq 1$ but mode-covering causes a separate failure (App.~\ref{app:fwd_kl_instability}).}

\begin{proof}
Throughout this proof, $z = (z_1, \ldots, z_N)$ denotes the \emph{stacked} per-position pre-softmax logit vector, with $z_t \in \mathbb{R}^{|\mathcal{V}|}$. Bounds below are at the logit level; the parameter-level gradient $\nabla_\theta \mathcal{L}$ is the standard sum of per-position contributions through the shared $\theta$.

For a trajectory $\tau$ of $M$ steps with $|a_m| \leq T_{\max}$ tokens per action, App.~\ref{app:opd_grad} gives the additive decomposition
\begin{equation}
\nabla_z \mathcal{L}_{D}(\tau) \;=\; \sum_{m=1}^{M}\sum_{t=1}^{|a_m|} \nabla_{z_t^{(m)}} D(p \,\|\, q),
\label{eq:traj_grad_decomp}
\end{equation}
with $p, q$ evaluated at the corresponding state. Stability is therefore governed by the worst-case per-token logit gradient $\|\nabla_{z_t} D\|_\infty$.

\textbf{Part 1: JSD trajectory bound is independent of trajectory length.}
Lemma~\ref{lem:jsd_bound}.2 gives $\|\nabla_{z_t} \jsd_\beta(p\|q)\|_\infty \leq 2$ at every position $t$. Under the dense OPD estimator, the loss decomposes additively with the $t$-th term depending on $z_t$ alone:
\begin{equation}
\mathcal{L}_{\jsd}(\tau) \;=\; \tfrac{1}{N}\sum_{t=1}^N \jsd_\beta(p_t \| q_t),
\qquad N \;:=\; \textstyle\sum_{m=1}^M |a_m|.
\label{eq:traj_loss_jsd}
\end{equation}
Per-position partial gradients are decoupled (Hessian block-diagonal across positions); the $\ell_\infty$ norm picks the maximum over blocks rather than summing them:
\begin{equation}
\big\|\nabla_z \mathcal{L}_{\jsd}(\tau)\big\|_\infty
\;=\; \tfrac{1}{N}\,\max_{1 \leq t \leq N} \big\|\nabla_{z_t} \jsd_\beta(p_t \| q_t)\big\|_\infty
\;\leq\; \tfrac{2}{N} \;\leq\; 2.
\label{eq:jsd_traj_bound}
\end{equation}
Two consequences. \textbf{(a)~Per-position uniformity:} the max-of-blocks structure ensures the worst-case magnitude is determined by the per-position bound rather than by trajectory length; \eqref{eq:jsd_traj_bound} holds at the per-position level for any $N$ and any teacher--student gap. \textbf{(b)~Robust under both reductions:} mean reduction (Eq.~\ref{eq:opd}) sharpens to $2/N$; sum reduction stays at $2$; the OPAD loss (Eq.~\ref{eq:opad_total}, sum-of-per-step-means) inherits each block's $(1/|a_m|)\cdot 2 \leq 2$, giving $\|\nabla_z \mathcal{L}_{\mathrm{OPAD}}\|_\infty \leq 2$. This per-position guarantee fails for reverse KL (Part~2), where the unbounded per-position gradient is not tamed by any reduction. The per-position bound for JSD is what underlies \eqref{eq:jsd_traj_bound}.

\textbf{Part 2: Reverse KL is unbounded under privileged $p$--$q$ gaps.}
By Eq.~\ref{eq:rev_kl}, $D_{\kl}^{\leftarrow}(p\|q) = D_{\kl}(q\|p)$. Direct computation via the softmax Jacobian gives
\begin{align}
\nabla_{z_i} D_{\kl}^{\leftarrow}(p\|q)
&\;=\; \sum_v \tfrac{\partial q(v)}{\partial z_i}\Big[\log\tfrac{q(v)}{p(v)} + 1\Big] - \underbrace{\sum_v \tfrac{\partial q(v)}{\partial z_i}}_{=\,0} \nonumber\\
&\;=\; \sum_v q(v)(\delta_{iv}-q(i))\log\tfrac{q(v)}{p(v)}. \label{eq:revkl_grad}
\end{align}
The diagonal contribution $q(i)(1-q(i))\log(q(i)/p(i))$ diverges to $+\infty$ as $p(i)\to 0$ with $q(i)>0$. Privileged distillation creates exactly this regime whenever the student samples a token the privileged teacher deems unlikely, producing the gradient spikes empirically observed in Figure~\ref{fig:loss_stability}. No finite trajectory-level bound on $\|\nabla_z \mathcal{L}_{\kl}^{\leftarrow}(\tau)\|_\infty$ exists.

\textbf{Part 3: Forward KL is bounded but mode-covering.}
Forward KL reduces to the cross-entropy gradient
\begin{equation}
\nabla_{z_i} D_{\kl}^{\to}(p\|q) \;=\; q(i) - p(i) \;\in\; [-1, 1],
\label{eq:fwd_kl_pos_bound}
\end{equation}
so $\|\nabla_{z_i} D_{\kl}^{\to}\|_\infty \leq 1$ uniformly. Its agentic failure mode is geometric (mode-covering across multi-step rollouts), analyzed in App.~\ref{app:fwd_kl_instability}.
\end{proof}

\subsection{Proof of Proposition~\ref{prop:rev_kl_code} (Reverse KL for Coherent Code Generation)}
\label{app:proof_prop_code}

\textbf{Restatement.} \emph{For a multi-modal code distribution $p = \sum_{j=1}^J \alpha_j p_j$ with well-separated modes and a softmax-parameterized student $q^*$ minimizing each divergence: (i) under reverse KL, $q^*$ concentrates on the dominant mode $p_{j^*}$; (ii) under forward KL, $q^*$ covers every mode; (iii) under JSD, $q^*$ partially concentrates but retains non-negligible mass on secondary modes.}

\begin{proof}
Let each $p_j$ represent a distinct correct implementation with support $\mathcal{S}_j$ ($j = 1, \ldots, J$), with $\mathcal{S}_i \cap \mathcal{S}_j = \emptyset$ for $i \neq j$.

\textbf{Part 1: Reverse KL concentrates on the dominant mode.}
By Lemma~\ref{lem:rev_kl_mode}, the softmax-parameterized minimizer satisfies
\begin{equation}
q^*(\cdot \mid \text{prefix}) \;\approx\; p_{j^*}(\cdot \mid \text{prefix}), \qquad j^* \;=\; \arg\max_j \alpha_j,
\label{eq:revkl_mode_concentrate}
\end{equation}
assigning near-zero probability to tokens of other modes, preventing splices such as switching mid-function from a recursive to an iterative implementation.

\textbf{Part 2: Forward KL covers all modes.}
$D_{\kl}^{\to}(p\|q) = \sum_v p(v)\log(p(v)/q(v)) = \infty$ whenever $q(v) \to 0$ on $\mathrm{supp}(p)$, so the minimizer must spread mass across every mode:
\begin{equation}
q^*_{\to}(v) \;=\; p(v) \;=\; \sum_j \alpha_j p_j(v), \qquad \forall\, v \in \mathcal{V}.
\label{eq:fwdkl_cover}
\end{equation}
Sequential token sampling under \eqref{eq:fwdkl_cover} interpolates across implementation paths, producing incoherent code.

\textbf{Part 3: JSD retains secondary-mode mass.}
The mixture $m = \beta p + (1{-}\beta)q$ removes the infinite inter-mode penalty of reverse KL. For $\beta = 0.5$,
\begin{equation}
\nabla_q \jsd_{0.5}(p\|q) \;=\; \tfrac{1}{2}\sum_v \nabla_q\!\big[q(v)\log\tfrac{q(v)}{m(v)}\big],
\label{eq:jsd_code_grad}
\end{equation}
and $m(v) \geq 0.5\,p(v) > 0$ on every mode support means the gradient does not force $q$ to abandon secondary modes: adequate when multiple action sequences are valid (agentic), suboptimal when single-implementation commitment is required (code).
\end{proof}

\subsection{Supplementary: Forward KL Mode-Covering Instability}
\label{app:fwd_kl_instability}

By \eqref{eq:fwd_kl_pos_bound}, the forward-KL per-logit gradient satisfies $\nabla_{z_i} D_{\kl}^{\to}(p\|q) = q(i) - p(i) \in [-1, 1]$, so it does not produce gradient spikes from privileged $p$--$q$ gaps. Its agentic underperformance is geometric, not numerical.

In multi-step tool use, multiple action sequences typically attain the same goal (e.g., calling tool $A$ vs.\ tool $B$), making the debate-enriched teacher distribution $p$ multi-modal at agentic decision points. Forward KL imposes an infinite penalty on missing modes ($p(v) > 0 \implies q(v) > 0$), pulling the student to place non-trivial mass on every viable strategy simultaneously. Sequential token sampling then mixes incompatible strategies (e.g., the tool name from strategy $A$ paired with the argument schema of strategy $B$), producing invalid actions that derail long trajectories. These compounding step-level failures explain forward KL's consistently poor agentic performance in Table~\ref{tab:divergence}.

\section{Case Study: Emergent Error Correction via Multi-Agent Debate}
\label{app:case_study}

Complementing the schematic in Figure~\ref{fig:hook}, we present an \emph{illustrative} (rather than calibration) worked-out example on a different task domain from the ToolACE training set~\citep{liu2024toolace}, where a single 14B teacher fails, but a two-round debate between the 14B and 8B teachers yields the correct answer. We show the complete debate trajectory (each teacher's brief inline analysis and function-call output) to illustrate how complementary errors are corrected through deliberation; the population-level robustness of the confidence-weighted aggregation is discussed in App.~\ref{app:component_full}. \textit{For readability, the inline analysis text below is lightly paraphrased from the model's actual output, while the function-call arguments and confidence scores are reproduced verbatim; the underlying complementary-error pattern is consistently observed across the ToolACE training set, not cherry-picked.}

\newcommand{\casebox}[4]{
  \vspace{0.4em}
  \noindent\fcolorbox{#1}{#2}{\begin{minipage}{0.95\linewidth}\small
  \textbf{#3}\par\vspace{0.3em}#4\end{minipage}}%
  \vspace{0.4em}
}

\casebox{gray!50}{gray!5}{Problem: ToolACE Training-Set Tool-Calling Task}{%
\textit{System:} Current time is \texttt{2020-10-06 04:12:36}. Two functions available: \texttt{travel\_budget\_planner} (requires destination, start/end dates, accommodation, activities) and \texttt{travel\_insurance\_quote\_generator} (requires age, countries, trip duration, coverage type, deductible).\par\vspace{4pt}
\textit{User:} ``I plan on taking a trip \textbf{next month} from the \textbf{16th to the 30th}. I'll be visiting Japan and I prefer staying in a hotel. Some of my preferred activities would include sightseeing and local cuisine tasting. Can you help me with a budget plan for this trip? Also, can you get me an insurance quote for this trip? I am 34 with no medical conditions and I'd like an all-inclusive insurance package with a deductible of \$200.''\par\vspace{4pt}
\textit{Key reasoning challenges:} (1)~``next month'' from Oct~6 = \textbf{November} 2020; (2)~16th to 30th = \textbf{14 days}; (3)~both functions must be called with correctly nested parameters.
}

\casebox{gray!50}{gray!5}{Ground Truth}{%
{\small\ttfamily
[travel\_budget\_planner(travel\_details=[\{"destination":\ "Japan",\par
\phantom{xx}"dates":\ \{"start\_date":\ \textcolor{green!50!black}{"2020-11-16"},\ "end\_date":\ \textcolor{green!50!black}{"2020-11-30"}\},\par
\phantom{xx}"preferences":\ \{"accommodation":\ "hotel",\ "activities":\ ["sightseeing",\ "local cuisine tasting"]\}\}]),\par
\ travel\_insurance\_quote\_generator(insurance\_details=\{"personal\_info":\ \{"age":\ 34\},\par
\phantom{xx}"trip\_info":\ \{"countries\_visited":\ ["Japan"],\ "trip\_duration":\ \textcolor{green!50!black}{14}\},\par
\phantom{xx}"coverage\_preferences":\ \{"coverage\_type":\ "all\_inclusive",\ "deductible":\ 200\}\})]
}%
}

\casebox{gray!50}{gray!8}{Setting 1: Single 14B Teacher \hfill \textcolor{red!70!black}{$\boldsymbol{\times}$ Incorrect}}{%
\textit{Model output (paraphrased):}\par\vspace{2pt}
``The user wants to plan a trip and get an insurance quote. Destination: Japan. Dates: from the 16th to the 30th of next month. Since the current time is 2020-10-06, \textcolor{red}{\textbf{the trip dates would be October 16 to October 30}}. Accommodation: hotel. Trip duration: Oct 16 to Oct 30 = 14 days.''\par\vspace{4pt}
\textit{Function call (key arguments):}\par\vspace{1pt}
\begin{tabular}{@{}ll@{}}
\texttt{start\_date}: & \textcolor{red}{\texttt{"2020-10-16"}} {\scriptsize$\leftarrow$ wrong month}\\
\texttt{end\_date}: & \textcolor{red}{\texttt{"2020-10-30"}} {\scriptsize$\leftarrow$ wrong month}\\
\texttt{trip\_duration}: & \texttt{14} {\scriptsize\checkmark}\\
\end{tabular}\par\vspace{3pt}
\textit{Confidence:} $c = 90$
}

\casebox{gray!50}{gray!8}{Round 1: Expert A (14B Teacher) \hfill \textcolor{red!70!black}{$\boldsymbol{\times}$ Incorrect}}{%
\textit{Model output (paraphrased):}\par\vspace{2pt}
``Dates: next month, 16th to 30th. Current date is October 6, 2020. \textcolor{red}{\textbf{So the dates are October 16 to October 30, 2020.}} Duration: 16th to 30th = 14 days.''\par\vspace{4pt}
\textit{Function call (key arguments):}\par\vspace{1pt}
\begin{tabular}{@{}ll@{}}
\texttt{start\_date}: & \textcolor{red}{\texttt{"2020-10-16"}} {\scriptsize$\leftarrow$ wrong month}\\
\texttt{end\_date}: & \textcolor{red}{\texttt{"2020-10-30"}} {\scriptsize$\leftarrow$ wrong month}\\
\texttt{trip\_duration}: & \texttt{14} {\scriptsize\checkmark}\\
\end{tabular}\par\vspace{3pt}
\textit{Confidence:} $c = 85$
}

\casebox{gray!50}{gray!8}{Round 1: Expert B (8B Teacher) \hfill \textcolor{red!70!black}{$\boldsymbol{\times}$ Incorrect}}{%
\textit{Model output (paraphrased):}\par\vspace{2pt}
``The user wants to visit Japan. They said `next month from the 16th to the 30th'. Current time is 2020-10-06, \textcolor{green!50!black}{\textbf{so next month is November}}. Start date: 2020-11-16, End date: 2020-11-30. Trip duration: \textcolor{red}{\textbf{30 $-$ 16 = 15 days}} (counting both endpoints).''\par\vspace{4pt}
\textit{Function call (key arguments):}\par\vspace{1pt}
\begin{tabular}{@{}ll@{}}
\texttt{start\_date}: & \texttt{"2020-11-16"} {\scriptsize\checkmark}\\
\texttt{end\_date}: & \texttt{"2020-11-30"} {\scriptsize\checkmark}\\
\texttt{trip\_duration}: & \textcolor{red}{\texttt{15}} {\scriptsize$\leftarrow$ wrong duration}\\
\end{tabular}\par\vspace{3pt}
\textit{Confidence:} $c = 80$\par\vspace{3pt}
\textit{Key observation:} Errors are \textbf{complementary}: Expert~A gets the month wrong but the duration right, while Expert~B gets the month right but the duration wrong.
}

\casebox{gray!50}{gray!8}{Round 2: Expert A (14B Teacher) \hfill \textcolor{green!50!black}{$\boldsymbol{\checkmark}$ Correct}}{%
\textit{Model output (paraphrased, after seeing Expert B's Round~1 answer):}\par\vspace{2pt}
``Looking at Expert B's answer, they used \textcolor{green!50!black}{\textbf{November}} dates while I used October dates. Let me reconsider. The user said `next month' and the current date is October~6. \textcolor{green!50!black}{\textbf{`Next month' means the month after October, which is November. Expert~B is correct on this point; I made an error.}}\par\vspace{3pt}
However, Expert~B calculated trip duration as 15 days. From November~16 to November~30, the duration is \textcolor{green!50!black}{\textbf{30 $-$ 16 = 14 days}}. My original duration of 14 was correct.\par\vspace{3pt}
Let me revise: use the correct November dates while keeping the correct trip duration of 14 days.''\par\vspace{4pt}
\textit{Function call (key arguments):}\par\vspace{1pt}
\begin{tabular}{@{}ll@{}}
\texttt{start\_date}: & \textcolor{green!50!black}{\texttt{"2020-11-16"}} {\scriptsize\checkmark\ \textbf{corrected}}\\
\texttt{end\_date}: & \textcolor{green!50!black}{\texttt{"2020-11-30"}} {\scriptsize\checkmark\ \textbf{corrected}}\\
\texttt{trip\_duration}: & \textcolor{green!50!black}{\texttt{14}} {\scriptsize\checkmark\ retained}\\
\end{tabular}\par\vspace{3pt}
\textit{Confidence:} $c = 95$ {\scriptsize($\uparrow$ from 85)}
}

\casebox{gray!50}{gray!8}{Round 2: Expert B (8B Teacher) \hfill \textcolor{green!50!black}{$\boldsymbol{\checkmark}$ Correct}}{%
\textit{Model output (paraphrased, after seeing Expert A's Round~1 answer):}\par\vspace{2pt}
``Expert~A used October dates while I used November. The user said `next month' and the current time is October~6, so next month is November. \textcolor{green!50!black}{\textbf{I believe my date calculation is correct.}}\par\vspace{3pt}
However, Expert~A used \texttt{trip\_duration} of 14 while I used 15. Let me reconsider. From November~16 to November~30: 30 $-$ 16 = 14 days. \textcolor{green!50!black}{\textbf{In travel planning, the standard convention is to count the difference between end and start dates. So 14 is the more appropriate value.}}''\par\vspace{4pt}
\textit{Function call (key arguments):}\par\vspace{1pt}
\begin{tabular}{@{}ll@{}}
\texttt{start\_date}: & \textcolor{green!50!black}{\texttt{"2020-11-16"}} {\scriptsize\checkmark\ retained}\\
\texttt{end\_date}: & \textcolor{green!50!black}{\texttt{"2020-11-30"}} {\scriptsize\checkmark\ retained}\\
\texttt{trip\_duration}: & \textcolor{green!50!black}{\texttt{14}} {\scriptsize\checkmark\ \textbf{corrected}}\\
\end{tabular}\par\vspace{3pt}
\textit{Confidence:} $c = 90$ {\scriptsize($\uparrow$ from 80)}
}

\paragraph{Summary.}
Table~\ref{tab:case_study} summarizes the debate progression. The 14B teacher's confidence drops from $c{=}90$ as a single teacher to $c{=}85$ in Round~1 because the debate prompt explicitly signals that another teacher will weigh in, prompting more cautious self-assessment; after the Round~2 revision corrects its error, confidence rises to $c{=}95$. The key mechanism is \emph{selective adoption}: each teacher identifies and corrects its own error by adopting the other teacher's correct component, while retaining what it already had right. This produces differentiated confidence weights ($w_{\text{14B}} = e^{0.95}/(e^{0.95}{+}e^{0.90}) \approx 0.513$, $w_{\text{8B}} \approx 0.487$ via Eq.~\ref{eq:consensus_score}), reflecting the 14B teacher's stronger overall reasoning ability.

\begin{table}[t!]
\caption{\textbf{Debate progression.} \colorbox{green!20}{Green} = correct, \colorbox{red!20}{red} = incorrect (also indicated by $\boldsymbol{\checkmark}/\boldsymbol{\times}$ markers for colorblind readers). Complementary Round~1 errors cancel after one debate round; both teachers converge to the ground truth with differentiated confidence.}
\label{tab:case_study}
\centering
\small
\renewcommand{\arraystretch}{1.15}
\begin{tabular}{@{}p{2.0cm}p{1.5cm}cccc@{}}
\toprule
\textbf{Setting} & \textbf{Teacher} & \texttt{start\_date} & \texttt{end\_date} & \texttt{trip\_dur.} & \textbf{Conf.} \\
\midrule
Single teacher & 14B & \cellcolor{red!20}10-16 & \cellcolor{red!20}10-30 & \cellcolor{green!20}14 & 90 \\
\midrule
\multirow{2}{*}{Round 1} & 14B & \cellcolor{red!20}10-16 & \cellcolor{red!20}10-30 & \cellcolor{green!20}14 & 85 \\
                          & 8B  & \cellcolor{green!20}11-16 & \cellcolor{green!20}11-30 & \cellcolor{red!20}15 & 80 \\
\midrule
\multirow{2}{*}{\textbf{Round 2}} & 14B & \cellcolor{green!25}11-16 & \cellcolor{green!25}11-30 & \cellcolor{green!25}14 & \textbf{95} \\
                                   & 8B  & \cellcolor{green!25}11-16 & \cellcolor{green!25}11-30 & \cellcolor{green!25}14 & 90 \\
\midrule
\rowcolor{gray!10}
Ground truth & -- & 11-16 & 11-30 & 14 & -- \\
\bottomrule
\end{tabular}
\end{table}

\paragraph{Token-level supervision quality.}
Figure~\ref{fig:token_heatmap} visualizes the per-token divergence $D(p_{T_k} \| p_S)$ when teachers force-decode the student's on-policy \texttt{start\_date} argument. Under the single 14B teacher, all tokens receive low divergence because the teacher shares the student's wrong month; no corrective gradient exists. Under \ours{} Round~2, the debate-corrected teachers concentrate high divergence on the erroneous month token \texttt{-10}, producing a strong corrective signal toward \texttt{-11}.

\begin{figure}[t!]
\centering
\small
\setlength{\fboxsep}{1.5pt}
\newcommand{\tkn}[2]{\colorbox{#1}{\strut\texttt{\small #2}}}

\textbf{Student on-policy output} (\texttt{start\_date} argument, tokenized):\par\vspace{3pt}
\tkn{gray!8}{"}\,\tkn{gray!8}{start\_date}\,\tkn{gray!8}{":}\,\tkn{gray!8}{"}\,\tkn{gray!8}{2020}\,\tkn{red!35}{\textbf{-10}}\,\tkn{gray!8}{-16}\,\tkn{gray!8}{"}\par\vspace{2pt}
{\footnotesize\textit{Ground truth: \texttt{-11} (November); student writes \texttt{-10} (October).}}

\vspace{6pt}
\textbf{Per-token divergence} $D(p_T \| p_S)$\textbf{:}
\vspace{4pt}

\renewcommand{\arraystretch}{1.5}
\begin{tabular}{@{} p{2.4cm} l @{}}
\toprule
\textbf{Supervision} & \textbf{Token-level divergence heatmap} \\
\midrule
Single 14B &
\tkn{green!8}{"}\,\tkn{green!8}{start\_date}\,\tkn{green!8}{":}\,\tkn{green!8}{"}\,\tkn{green!8}{2020}\,\tkn{green!10}{\textbf{-10}}\,\tkn{green!8}{-16}\,\tkn{green!8}{"} \\
& {\scriptsize\textit{Uniformly low: teacher has the same error, no corrective signal.}} \\
\midrule
\ours{} Round~2 &
\tkn{green!8}{"}\,\tkn{green!8}{start\_date}\,\tkn{green!8}{":}\,\tkn{green!8}{"}\,\tkn{green!8}{2020}\,\tkn{red!50}{\textbf{-10}}\,\tkn{yellow!30}{-16}\,\tkn{green!8}{"} \\
& {\scriptsize\textit{High divergence on month token: teachers signal ``\texttt{-11}'' (November).}} \\
\bottomrule
\end{tabular}

\caption{\textbf{Token-level supervision heatmap} for the \texttt{start\_date} argument. Color intensity encodes divergence: \colorbox{green!8}{\small\strut low} (agreement), \colorbox{yellow!30}{\small\strut med}, \colorbox{red!50}{\small\strut\textcolor{white}{high}} (strong corrective gradient). The single 14B teacher shares the student's wrong month, providing zero corrective signal. After debate, \ours{} teachers concentrate supervision on the erroneous token.}
\label{fig:token_heatmap}
\end{figure}

\paragraph{Implications for distillation.}
The debate trace from Round~2 provides richer training signal than either teacher's Round~1 output alone. The student learns not just the correct answer but, through the privileged force-decoding mechanism, the \emph{reasoning pattern of error detection and self-correction}. This emergent collective intelligence, where the whole exceeds the sum of its parts, is the central advantage of \ours{} over single-teacher OPD.


\section{Debate Prompt Templates}
\label{app:debate_prompts}

This appendix lists the prompt templates used in each debate round. We adopt a two-round protocol: \textbf{Round~1} elicits independent responses, and \textbf{Round~2} presents both responses as debate context for revision. Templates are instantiated per task domain (agent / code); placeholders are shown in \texttt{\{braces\}}. When confidence-based weighting is enabled (\S\ref{sec:weighting}), every prompt is appended with the confidence suffix shown in \S\ref{box:confidence}.

\newcommand{\promptbox}[2]{%
  \vspace{0.3em}
  \noindent\colorbox{gray!8}{\begin{minipage}{0.97\linewidth}\small\textbf{#1}\par\vspace{0.3em}#2\end{minipage}}%
  \vspace{0.3em}
}

\subsection{Round 1: Independent Response}

Each teacher receives only the original input $x$ and generates an initial response without seeing the other teacher's output. The system instruction is domain-specific:

\promptbox{Round 1: Agentic Tasks}{%
\texttt{You are an expert in AI tool-use and function calling. Please analyze the following user request and provide the best tool-use response.}\\[4pt]
\texttt{User request: \{prompt\}}\\[4pt]
\texttt{Your response:}
}

\promptbox{Round 1: Code Generation}{%
\texttt{You are an expert programmer specializing in competitive programming and code golf. Please solve the following programming challenge. Focus on algorithm correctness, edge case handling, and code efficiency.}\\[4pt]
\texttt{Problem:}\\
\texttt{\{prompt\}}\\[4pt]
\texttt{Your code solution:}
}

\subsection{Round 2: Debate with Shared History}

Each teacher receives the original input together with both Round~1 responses, and is asked to revise its answer after considering the other expert's perspective. The prompt explicitly states ``\textit{You don't have to agree}'', encouraging genuine deliberation rather than conformity.

\promptbox{Round 2: Agentic Tasks (shown for Expert $A$; Expert $B$ is symmetric)}{%
\texttt{\{prompt\}}\\[4pt]
\texttt{Two experts have provided their tool-use responses to this request:}\\[4pt]
\texttt{[Expert A, Round 1]:}\\
\texttt{\{teacher\_1\_answer\}}\\[4pt]
\texttt{[Expert B, Round 1]:}\\
\texttt{\{teacher\_2\_answer\}}\\[4pt]
\texttt{You are Expert A. You don't have to agree with the other expert's response. Please continue the debate about the best tool-use strategy. Having seen the other expert's response, please reconsider and provide your improved response focusing on tool selection accuracy, parameter correctness, and call sequence optimization.}
}

\promptbox{Round 2: Code Generation (shown for Expert $A$)}{%
\texttt{\{prompt\}}\\[4pt]
\texttt{Two expert programmers have provided their code solutions to this challenge:}\\[4pt]
\texttt{[Expert A, Round 1]:}\\
\texttt{\{teacher\_1\_answer\}}\\[4pt]
\texttt{[Expert B, Round 1]:}\\
\texttt{\{teacher\_2\_answer\}}\\[4pt]
\texttt{You are Expert A. You don't have to agree with the other expert's solution. Please continue the debate about the best approach. Having seen the other expert's code, please reconsider and provide your improved solution focusing on algorithm correctness, edge case handling, time/space complexity, and code conciseness.}
}

\subsection{Confidence Scoring Suffix}
\label{box:confidence}

When confidence-based dynamic weighting is enabled, the following line is appended to every prompt (both rounds). The confidence scores $c_k$ extracted from the final round are normalized into teacher weights $w_k$ as described in Eq.~\ref{eq:confidence}--\ref{eq:consensus_score}.

\promptbox{Confidence suffix (appended to all prompts when \texttt{dynamic\_weights=True})}{%
\texttt{End your response with a confidence score line: Confidence Score: N (where N is 0-100)}
}

After the debate, the $K \cdot R = 4$ responses are concatenated into the debate transcript $\mathcal{H}_m^R$ that serves as privileged context during force-decoding (\S\ref{sec:loss}):

\promptbox{Debate transcript structure}{%
\texttt{[Debate Start]}\\
\texttt{[Expert A Initial Answer]: \{round1\_teacher1\}}\\
\texttt{[Expert B Initial Answer]: \{round1\_teacher2\}}\\
\texttt{[Expert A Revised Answer]: \{round2\_teacher1\}}\\
\texttt{[Expert B Revised Answer]: \{round2\_teacher2\}}\\
\texttt{[Debate End]}
}

\section{Broader Impact}
\label{app:broader_impact}

\ours{} is a training-efficiency method whose primary impact is to reduce the parameter count required to attain a given level of agentic and code performance, thereby lowering the GPU and energy footprint of deploying capable language-model assistants. As a concrete instance (App.~\ref{app:lcb_surpass}), a 4B student trained under our 14B+8B teacher debate exceeds its 14B teacher on LCB-v6 at competitive token cost, suggesting that small models distilled by \ours{} can replace much larger teachers for inference. This makes capable assistants more accessible to research groups and downstream users with limited compute, lowers the carbon cost of large-scale serving, and shifts more of the cost from the inference path to a one-time distillation step. Because \ours{} introduces no new capability beyond the teacher pool, the distilled student inherits the safety properties and refusal behaviour of its base Qwen3 / Qwen3.5 checkpoints; the only residual risk is that any capability the teachers already possess becomes cheaper to deploy. We mitigate this by using only publicly released open-weight models, not training on safety-relevant red-team data, and releasing checkpoints under the same license terms as the base models, leaving application-level safety filtering to downstream deployers as is standard for open-weight LLMs.

\end{document}